\pdfoutput=1
\documentclass[10pt,twocolumn,letterpaper]{article}
\usepackage[letterpaper,left=0.75in,right=0.75in,top=0.8in,bottom=0.85in,columnsep=0.3in,footskip=0.35in]{geometry}
\usepackage[T1]{fontenc}
\usepackage{amsmath}
\usepackage{amsthm}
\usepackage{mathtools}
\usepackage{newtxtext}
\usepackage{newtxmath}
\usepackage[varqu,varl,scaled=0.95]{zi4}
\usepackage[hyphens]{url}
\usepackage{graphicx}
\usepackage{xcolor}
\usepackage{natbib}
\usepackage{caption}
\usepackage{placeins}
\usepackage{algorithm}
\usepackage{algorithmic}
\usepackage{array}
\usepackage{booktabs}
\usepackage{multirow}
\usepackage{titlesec}
\titleformat{\section}{\large\bfseries}{\thesection}{0.6em}{}
\titlespacing*{\section}{0pt}{2.0ex plus 0.6ex minus 0.3ex}{1.0ex plus 0.2ex}
\titleformat{\subsection}{\normalsize\bfseries}{\thesubsection}{0.5em}{}
\titlespacing*{\subsection}{0pt}{1.6ex plus 0.5ex minus 0.3ex}{0.6ex plus 0.1ex}
\titleformat{\subsubsection}{\normalsize\bfseries\itshape}{\thesubsubsection}{0.5em}{}
\titlespacing*{\subsubsection}{0pt}{1.3ex plus 0.4ex minus 0.2ex}{0.5ex plus 0.1ex}
\titleformat{\paragraph}[runin]{\normalsize\bfseries}{}{0pt}{}
\titlespacing*{\paragraph}{0pt}{1.0ex plus 0.4ex minus 0.2ex}{0.6em}
\usepackage{microtype}
\definecolor{linkblue}{RGB}{28,58,140}
\usepackage{hyperref}
\hypersetup{colorlinks=true,linkcolor=linkblue,citecolor=linkblue,urlcolor=linkblue,
  pdftitle={PXtal: Learning to Align Powder X-Ray Diffraction and Crystal Structures under Information Asymmetry across Modalities},
  pdfauthor={Zhuoran Yang, Christopher M. Collins, Bei Peng, Luke M. Daniels, Matthew J. Rosseinsky, Vladimir V. Gusev}}
\newcommand{\appref}[1]{\hyperref[#1]{Appendix~\ref*{#1}}}

\theoremstyle{plain}
\newtheorem{theorem}{Theorem}[section]
\newtheorem{proposition}[theorem]{Proposition}

\newtheorem{corollary}[theorem]{Corollary}
\theoremstyle{definition}

\theoremstyle{remark}

\newtheorem*{propositionrestated}{Proposition \ref{thm:gkl_decomposition} (restated)}
\newtheorem*{corollaryrestated}{Corollary \ref{cor:conditional_ce} (restated)}

\title{\textbf{PXtal: Learning to Align Powder X-Ray Diffraction and Crystal Structures under Information Asymmetry across Modalities}}
\author{
Zhuoran Yang$^{1}$,\enskip
Christopher M. Collins$^{2}$,\enskip
Bei Peng$^{3}$,\enskip
Luke M. Daniels$^{4}$,\\
Matthew J. Rosseinsky$^{2,4}$,\enskip
Vladimir V. Gusev$^{1}$\\[0.6em]
{\small $^{1}$School of Computer Science and Informatics, University of Liverpool, Liverpool, UK}\\
{\small $^{2}$Leverhulme Research Centre for Functional Materials Design, Materials Innovation Factory, University of Liverpool, Liverpool, UK}\\
{\small $^{3}$School of Computer Science, University of Sheffield, Sheffield, UK}\\
{\small $^{4}$Department of Chemistry, University of Liverpool, Liverpool, UK}
}
\date{}

\begin{document}
\maketitle
\begingroup
\renewcommand\thefootnote{}
\footnotetext{\raggedright Preprint. Under review. Correspondence to: \mbox{Matthew J. Rosseinsky} \textless\href{mailto:M.J.Rosseinsky@liverpool.ac.uk}{M.J.Rosseinsky@liverpool.ac.uk}\textgreater, \mbox{Vladimir V. Gusev} \textless\href{mailto:Vladimir.Gusev@liverpool.ac.uk}{Vladimir.Gusev@liverpool.ac.uk}\textgreater.}
\addtocounter{footnote}{-1}
\endgroup

\begin{abstract}
Scientific multimodal learning commonly assumes that paired views are
comparably informative. Powder X-ray diffraction (PXRD) makes this mismatch
explicit: compressing a three-dimensional crystal structure into a
one-dimensional diffraction pattern loses information and makes the pattern
harder to connect to the crystal structure that produced it. We introduce
\textbf{PXtal}, a framework for learning aligned PXRD and crystal
representations under this physically imposed information asymmetry.
PXtal uses Unbalanced Optimal Transport (UOT) to adapt the cross-modal coupling
and coupling-level generalized Kullback--Leibler (GKL) divergence to supervise
the full transport plan.
Across six test sets,
including four zero-shot transfer sets, PXtal consistently
outperforms the baseline models in PXRD-to-crystal candidate retrieval, with
the largest gains when PXRD patterns have close but crystallographically
distinct nonpaired neighbors, meaning similar input patterns associated with
different crystals. The resulting crystal and PXRD encoders transfer more
effectively to downstream materials and crystallographic tasks.
These results identify information asymmetry as a general design problem in
scientific multimodal learning: alignment objectives should reflect what each
modality preserves.
\end{abstract}

\section{Introduction}

Powder X-ray diffraction (PXRD) probes crystalline materials by measuring
scattered X-ray intensity as a function of diffraction angle. It is widely used
to identify crystalline phases and assess synthesis products, yet interpreting
the resulting patterns is often challenging and time consuming
\citep{gulay2024navigation}. Constructive scattering from the many small,
randomly oriented crystallites in a powder
produces Bragg peaks corresponding to characteristic interplanar separations in
the structure; their intensities depend on the atoms present and their arrangement
\citep{dinnebier2008principles,young1993rietveld}.
Each allowed reflection has three Miller indices $(h,k,l)$ that label a family
of lattice planes. An idealized intensity profile for a single phase can then be
written as
\begin{equation}
  I(2\theta)
  \approx b(2\theta)
  +\sum_{hkl}w_{hkl}|F_{hkl}|^2
  \phi(2\theta-2\theta_{hkl}).
  \label{eq:pxrd_profile}
\end{equation}
Here $I(2\theta)$ is the intensity at scattering angle $2\theta$ and
$b(2\theta)$ is the background. Bragg's law,
$2d_{hkl}\sin\theta_{hkl}=\lambda$, links the position of reflection $hkl$ to
the plane spacing $d_{hkl}$ for X-ray wavelength $\lambda$. The complex
structure factor $F_{hkl}$ combines the scattering contributions from atoms in
the unit cell; its phase arises mathematically from their positions. The weight
$w_{hkl}$ collects geometric and experimental factors together with reflection
multiplicity, the number of symmetry-equivalent reflections and thus planes
associated with $hkl$. The function $\phi$ describes the peak profile.
\appref{app:pxrd} gives the expanded relation and defines each term.

\begin{figure}[t]
  \centering
  \includegraphics[width=\columnwidth]{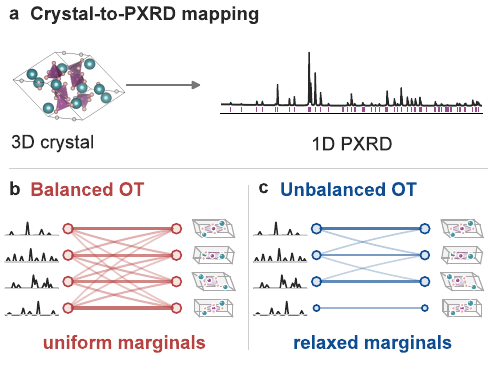}
\caption{PXtal overview. (a) The crystal-to-PXRD map creates physically imposed
  information asymmetry. In (b) and (c), each node represents an encoded PXRD
  pattern or crystal structure; line width and node area encode coupling and
  marginal mass. (b) Balanced OT uses uniform marginals and a probability
  measure. (c) UOT relaxes the marginals and yields a nonnegative measure.}
  \label{fig:pxrd_measurement}
\end{figure}

The measured intensity profile contains the contribution $|F_{hkl}|^2$ but not
the complex phase of $F_{hkl}$. Diffraction from an individual crystallite is
direction resolved; powder averaging collapses these directions into one
angular coordinate, and finite resolution merges nearby reflections. Although
a crystal structure can be projected into a PXRD pattern, the lost scattering
information makes an observed pattern more difficult to connect to its
underlying structure in three dimensions. Conventional analysis therefore first
narrows candidate phases through database search and match \citep{altomare2008} and
uses indexing to establish compatible unit-cell geometry and peak positions
\citep{pawley1981}. A candidate structural model is then refined against the
whole profile by Rietveld analysis
\citep{rietveld1969profile,toby2013}.

These effects leave the paired views with unequal information. A crystal description
specifies atom types, fractional coordinates, lattice parameters, and symmetry in
three dimensions, whereas its PXRD pattern records only the aggregate reflection
intensities in Eq.~\eqref{eq:pxrd_profile} after phase loss, powder averaging,
and peak overlap. We refer to this mismatch as \emph{information asymmetry across
modalities}. Related asymmetries can produce
modality gaps and information imbalance
\citep{liang2022mind,schrodi2025two}, as well as failures of concept association
and compositional binding
\citep{gurung2025common,yuksekgonul2023when,tang2023lemons}. PXRD provides a
physically grounded instance with an observable consequence: some patterns are
distinctive, whereas others have close nonpaired neighbors.

Machine learning methods for PXRD classify phases and symmetry
\citep{lee2020deep,wang2020rapid}, support probabilistic interpretation
\citep{szymanski2021probabilistic}, and benchmark
classification on simulated patterns
\citep{oviedo2019fast,cao2025simxrd}. More directly, XtalNet
\citep{lai2025end} and PXRDGen \citep{li2025powder} first align PXRD and crystal
embeddings with the contrastive InfoNCE objective. Their PXRD embeddings then condition generative
structure models. PXRDnet \citep{guo2025ab} takes a different route: it
generates candidate structures from PXRD and chemical formulae through a latent
crystal model coupled to a PXRD regressor. All three models use PXRD
representations learned from paired crystal structures to guide crystal
structure generation, but do not focus on the joint representation problem.
We formulate this problem as cross-modal alignment under physically imposed
information asymmetry. Because crystallographically distinct structures can
produce similar PXRD patterns, the same pair confidence is not equally
supported for every sample. InfoNCE nevertheless assigns every pair the same
target confidence and treats all nonpaired batch members as negatives,
regardless of whether their PXRD patterns are distant or close.

To address this mismatch, we introduce \textbf{PXtal}, a framework that learns
PXRD and crystal encoders in a shared representation space
(Fig.~\ref{fig:pxrd_measurement}). PXtal forms a batch-level transport coupling
and lets its marginals adapt across samples to the strength of competing
matches. Our primary model, \textbf{PXtal-UOT$_{\mathrm{GKL}}$}, treats this
coupling as a nonnegative measure and directly applies coupling-level
generalized Kullback--Leibler (GKL) supervision
\citep{peyre2019computational,sejourne2023unbalanced}.
For pretraining at scale, \textbf{PXtal-UFGOT} adds fused Gromov Wasserstein
terms that preserve relations within each modality
\citep{vayer2019optimal,sejourne2021unbalancedgw}.

We evaluate PXtal on three tasks.
\textbf{PXRD-to-crystal candidate retrieval.} Given an input PXRD pattern, the
pretrained cross-modal similarity ranks candidate crystal structures,
supporting database search and match before detailed confirmation and
refinement \citep{altomare2008}. \textbf{Crystal encoder evaluation.} The
crystal encoder is used to predict band gap, formation energy, metallicity,
and bulk and shear moduli, testing transfer to electronic, energetic, and elastic
properties used in materials screening \citep{dunn2020benchmarking}.
\textbf{PXRD encoder evaluation.} The PXRD encoder is used to predict crystal
system, space group, and unit-cell geometry from diffraction alone, providing
symmetry hypotheses and initial lattice estimates for indexing and structure solution
\citep{park2017classification,chitturi2021automated}.
\textbf{In relation to autonomous materials workflows.}
Autonomous laboratories couple robotic experiments to algorithmic decisions
\citep{burger2020mobile,szymanski2023autonomous,szymanski2023adaptive}. In
inorganic synthesis, automating these tasks links PXRD analysis to the experimental
feedback loop: phase identification returns synthesis outcomes to the planning
system, while online analysis can steer subsequent diffraction measurements.

\paragraph{Contributions.}
\textbf{(1) Problem formulation.} We identify joint PXRD and crystal representation
learning as cross-modal alignment under physically imposed information
asymmetry. \textbf{(2) Adaptive alignment.}
PXtal-UOT$_{\mathrm{GKL}}$ treats the transport plan as a nonnegative measure,
adapts its marginals with UOT, and directly supervises the full plan using
coupling-level GKL. PXtal-UFGOT extends this objective with within-modality
relations. \textbf{(3) Evaluation and transfer.} We conduct the first joint
PXRD and crystal representation pretraining study at million-pair scale, using
up to 2.05 million pairs, and evaluate PXRD-to-crystal retrieval across six
test sets, including four zero-shot transfer sets. PXtal improves retrieval
most when PXRD patterns have close but crystallographically distinct nonpaired
neighbors, and both encoders transfer effectively to downstream tasks.

\section{Related Work and Preliminaries}

\subsection{Contrastive Multimodal Learning}
\label{subsec:contrastive_related}

Contrastive learning separates paired observations from alternatives in
continuous representation spaces. Predictive objectives yielded InfoNCE
\citep{oord2018representation}, while instance discrimination enabled scalable
visual pretraining \citep{wu2018unsupervised,he2020momentum,chen2020simple}.

For paired modalities, let $\mathbf z_i^x$ and $\mathbf z_i^y$ be the
normalized encodings of paired observations. One directional InfoNCE term is
\[
  \mathcal L_{x\to y}
  =
  -\frac{1}{N}\sum_{i=1}^{N}
  \log
  \frac{\exp(s(\mathbf z_i^x,\mathbf z_i^y)/\tau)}
       {\sum_{j=1}^{N}\exp(s(\mathbf z_i^x,\mathbf z_j^y)/\tau)},
\]
where $N$ is the batch size, $s$ is the cross-modal similarity, and
$\tau>0$ is the contrastive temperature. CLIP combines this term with its
reverse direction for paired image--text learning
\citep{radford2021learning}. Subsequent work scales multimodal pretraining or
replaces the training objective
\citep{jia2021scaling,zhai2023sigmoid}.

Recent analyses describe information asymmetry across modalities at
complementary levels. Modality gap captures separation at the representation
level \citep{liang2022mind}; information imbalance links such separation to
object bias \citep{schrodi2025two}. Concept association bias and weak binding
of objects to attributes, relations, and order provide behavioral and semantic
manifestations
\citep{gurung2025common,yuksekgonul2023when,tang2023lemons}.

\subsection{Optimal Transport and Inverse Alignment}
\label{subsec:classical_ot}

Optimal transport compares discrete probability measures with weights
$\mathbf a\in\Delta_N$ and $\mathbf b\in\Delta_M$ through a probability
coupling with prescribed marginals. The classical problem is a linear program
\citep{kantorovitch1958translocation,peyre2019computational};
Benamou and Brenier give an equivalent dynamic formulation
\citep{benamou2000computational}. The admissible couplings are
\[
  \mathcal U(\mathbf a,\mathbf b)
  =
  \{\mathbf P\in\mathbb R_+^{N\times M}:
    \mathbf P\mathbf 1_M=\mathbf a,
    \mathbf P^\top\mathbf 1_N=\mathbf b\},
\]
where $\mathbf 1_n$ is the length-$n$ all-ones vector and
$\Delta_n=\{\mathbf w\in\mathbb R_+^n:
\mathbf 1_n^\top\mathbf w=1\}$ is the probability simplex.
For a cost matrix $\mathbf C$ and entropy weight $\varepsilon>0$, entropic OT solves
\begin{equation}
  \mathbf P^{\mathrm{OT}}
  =
  \arg\min_{\mathbf P\in\mathcal U(\mathbf a,\mathbf b)}
  \langle\mathbf C,\mathbf P\rangle
  -\varepsilon\,\mathcal H(\mathbf P),
  \label{eq:eot_prelim}
\end{equation}
where $\mathcal H(\mathbf P)=-\sum_{ij}P_{ij}(\log P_{ij}-1)$ is the coupling
entropy, with $0\log0:=0$, and Sinkhorn scaling enforces the prescribed marginals
\citep{sinkhorn1967diagonal,peyre2019computational}.
Inverse optimal transport learns costs from observed assignments
\citep{li2019learning,stuart2020inverse}. The corresponding contrastive
implementations use conditional cross entropy to supervise normalized couplings
\citep{shi2023understanding,shi2024ot}.

\section{Methods}
\label{sec:methods}

\paragraph{Overview.}
Each batch of paired PXRD patterns and crystal structures yields a learned cost
matrix across modalities. \textbf{PXtal-CLIP} applies symmetric InfoNCE to the
two encoder outputs. \textbf{PXtal-OT} replaces the separate conditional
normalizations with one balanced batch coupling. \textbf{PXtal-UOT} then
relaxes the fixed marginals, allowing coupling mass to vary across samples.
Applying conditional cross entropy to this plan gives
\textbf{PXtal-UOT$_{\mathrm{CE}}$}; applying GKL to the full nonnegative
measure gives \textbf{PXtal-UOT$_{\mathrm{GKL}}$}.
\textbf{PXtal-UFGOT} adds within-modality relational distortion to the same
unbalanced coupling, with \textbf{PXtal-FGOT} as its balanced counterpart.

\subsection{PXtal Variants}
\label{subsec:PXtal}

\begin{table*}[!t]
  \centering
  {\small
  \setlength{\tabcolsep}{1mm}
  \begin{tabular*}{\textwidth}{@{\extracolsep{\fill}}llrrrr@{}}
    \toprule
    \textbf{Model} & \textbf{Objective} & \textbf{Epochs} & \textbf{Mean (s/epoch)}
      & \textbf{Median (s/epoch)} & \textbf{Overhead} \\
    \midrule
    PXtal-CLIP & InfoNCE & 100 & 133.93 & 133.59 & 0.00\% \\
    PXtal-OT & OT & 100 & 134.65 & 134.03 & 0.54\% \\
    PXtal-UOT$_{\mathrm{GKL}}$ & UOT & 100 & 134.89 & 134.70 & 0.72\% \\
    PXtal-UFGOT & UFGOT & 100 & 137.87 & 137.31 & 2.95\% \\
    PXtal-FGOT & FGOT & 100 & 141.85 & 141.46 & 5.92\% \\
    \bottomrule
  \end{tabular*}
  }
  \caption{Matched CrystaLLM-v1 pretraining time under the same four-node,
  16-GPU setup and 100-epoch schedule. Mean and median report wall-clock
  seconds per epoch; overhead is relative to PXtal-CLIP. Full training
  settings are in \appref{app:repro}.}
  \label{tab:runtime}
\end{table*}

\subsubsection{Common Setup}

Given paired PXRD patterns and crystal structures
$\{(x_i^{P},x_i^{C})\}_{i=1}^{N}$, the PXRD Transformer $f_{\theta}$ and
crystal graph neural network $g_{\psi}$ produce $d$-dimensional embeddings
$\mathbf z_i^{P}=f_{\theta}(x_i^{P})$ and
$\mathbf z_i^{C}=g_{\psi}(x_i^{C})$. After
$\ell_2$ normalization, their cosine similarity is
$s^{*}_{ij}=\langle\mathbf z_i^{P},\mathbf z_j^{C}\rangle$, and the transport cost is
$C^{\Theta}_{ij}=1-s^{*}_{ij}$ for $\Theta=\{\theta,\psi\}$. Unless stated
otherwise, $\tilde{\mathbf P}=N^{-1}\mathbf I_N$ and
$\mathbf a=\mathbf b=N^{-1}\mathbf 1_N$.
PXtal-CLIP applies the symmetric InfoNCE objective from
Sec.~\ref{subsec:contrastive_related} to these embeddings.

\subsubsection{Balanced Optimal Transport: PXtal-OT}
\label{subsubsec:pxtal-ot}

PXtal-OT uses the same learned similarities as PXtal-CLIP but replaces the
separate row and column normalizations with one balanced coupling. It uses
the coupling
$\mathbf P_{\Theta}^{\mathrm{OT}}$ from Eq.~\eqref{eq:eot_prelim} and minimizes
$\mathrm{KL}(\tilde{\mathbf P}\|\mathbf P_{\Theta}^{\mathrm{OT}})$. PXtal-OT
therefore learns a joint batch coupling under prescribed uniform row and
column marginals. It provides the balanced reference for testing whether
sample-dependent marginal mass improves alignment.

\subsubsection{Unbalanced Optimal Transport Coupling}
\label{subsubsec:pxtal-uot}

Balanced OT fixes every row and column mass at $N^{-1}$. The UOT variants
replace these equalities with GKL penalties under entropy regularization
\citep{chizat2018scaling,peyre2019computational,
sejourne2023unbalanced}. UOT optimizes
$\mathbf P\in\mathbb R_+^{N\times N}$ as a nonnegative measure over PXRD and
crystal sample pairs, so its marginals and total mass can vary. For
nonnegative vectors, $D_{\mathrm{GKL}}$ denotes the generalized KL divergence,
the $\beta=1$ member of the beta divergence family
\citep{fevotte2011algorithms}; \appref{app:pxrd} gives its extended-value
definition. Given the learned cost $\mathbf C^\Theta$ and prior marginals
$\mathbf a,\mathbf b$, both UOT variants compute
\begin{equation}
  \begin{aligned}
    \mathbf P_{\Theta}^{\mathrm{UOT}}
    &=\arg\min_{\mathbf P\geq 0}\ \Bigl\{
      \langle\mathbf C^{\Theta},\mathbf P\rangle
      -\varepsilon\,\mathcal H(\mathbf P) \\
    &\quad+\alpha_{0}\,D_{\mathrm{GKL}}(\mathbf P\mathbf 1_N\|\mathbf a)
      +\alpha_{1}\,D_{\mathrm{GKL}}(\mathbf P^{\top}\mathbf 1_N\|\mathbf b)\Bigr\}.
  \end{aligned}
  \label{eq:ueot_inner}
\end{equation}
The positive coefficients $\alpha_0$ and $\alpha_1$ penalize deviations of
the learned row and column marginals from their priors, with balanced OT
recovered as the penalties tend to infinity. At finite penalties, the row
marginal $\mathbf u=\mathbf P\mathbf 1_N$, the column marginal
$\mathbf v=\mathbf P^\top\mathbf 1_N$, and the total mass are outputs of the
inner optimization. PXtal-UOT$_{\mathrm{CE}}$ applies conditional CE after
normalizing this plan, whereas PXtal-UOT$_{\mathrm{GKL}}$ applies an outer GKL
loss to the full measure. Section~\ref{sec:uot_gkl} shows which marginal signal
conditional normalization removes.

\subsubsection{Relational Extension: PXtal-UFGOT}
\label{subsubsec:pxtal-ufgot}

PXtal-UOT$_{\mathrm{GKL}}$ aligns paired PXRD patterns and crystal structures
through a direct cost across modalities. PXtal-UFGOT uses the same
coupling-level GKL supervision and also models relational geometry within each
modality. This extension follows
distributional views of
contrastive learning \citep{shi2023understanding,shi2024ot} and
fused Gromov Wasserstein transport
\citep{memoli2011gromov,peyre2016gromov,vayer2019optimal,sejourne2021unbalancedgw},
which combine a direct feature cost with relational structure.
Let $\mathbf D^P$ and $\mathbf D^C$ denote pairwise cosine similarities within
the two modality batches. Following the similarity-based structural
regularizer of \citet{xu2023regularized}, we use the bilinear score
\begin{equation}
  \mathcal B(\mathbf P;\mathbf D^P,\mathbf D^C)
  =
  \left\langle
    \mathbf P,\mathbf D^P\mathbf P\mathbf D^C
  \right\rangle .
  \label{eq:fgot_gw_term}
\end{equation}
This score rewards couplings that preserve relational neighborhoods across
PXRD and crystal representations. Since
$\nabla_{\mathbf P}\mathcal B(\mathbf P;\mathbf D^P,\mathbf D^C)
=2\mathbf D^P\mathbf P\mathbf D^C$, we absorb the factor of two into the
nonnegative coefficient $\rho$ and set
$\mathbf C_{\mathrm{rel}}^{(r)}
=\mathbf C^\Theta-\rho\mathbf D^P\mathbf P^{(r)}\mathbf D^C$.
Starting from $\mathbf P^{(0)}=\mathbf a\mathbf b^\top$, one proximal step is
\begin{equation}
  \begin{aligned}
    \mathbf P^{(r+1)}
    &=\arg\min_{\mathbf P\geq 0}\ \Bigl\{
      \langle\mathbf C_{\mathrm{rel}}^{(r)},\mathbf P\rangle
      -\varepsilon\,\mathcal H(\mathbf P) \\
    &\qquad+\alpha_0\,D_{\mathrm{GKL}}(\mathbf P\mathbf 1_N\|\mathbf a)
      +\alpha_1\,D_{\mathrm{GKL}}(\mathbf P^\top\mathbf 1_N\|\mathbf b)\\
    &\qquad+\tau_{\mathrm{prox}}\,
      D_{\mathrm{GKL}}(\mathbf P\|\mathbf P^{(r)})\Bigr\}.
  \end{aligned}
  \label{eq:ufgot_inner}
\end{equation}
We unroll a fixed number of the log-domain updates of
\citet{xu2023regularized}. PXtal-UOT is the
$\rho=\tau_{\mathrm{prox}}=0$ case with one outer step. The balanced ablation
\textbf{PXtal-FGOT} uses the same proximal outer step with prescribed uniform
marginals. \appref{app:ot} gives the unified outer update, inner solvers,
and solver settings.
We use PXtal-UFGOT when batch geometry is informative at scale;
PXtal-UOT$_{\mathrm{GKL}}$ remains the default direct coupling objective.

\subsubsection{From Conditional CE to GKL at the Coupling Level}
\label{sec:uot_gkl}

Conditional CE is a natural outer loss for PXtal-UOT, but it first normalizes
the predicted plan and therefore discards learned marginal mass.
PXtal-UOT$_{\mathrm{GKL}}$ instead treats the target and predicted plans as
nonnegative measures and applies GKL entrywise. The following chain rule makes
the removed term explicit.

\begin{proposition}[GKL decomposition for nonnegative couplings]
  \label{thm:gkl_decomposition}
  Let $\mathbf Q,\mathbf P\in\mathbb R_+^{N\times M}$ have strictly positive
  row sums $\mathbf r=\mathbf Q\mathbf 1_M$ and
  $\mathbf u=\mathbf P\mathbf 1_M$, and assume $P_{ij}>0$ whenever $Q_{ij}>0$.
  Define row conditionals
  $q(j\mid i)=Q_{ij}/r_i$ and $p(j\mid i)=P_{ij}/u_i$. Then
  \begin{equation}
    D_{\mathrm{GKL}}(\mathbf Q\|\mathbf P)
    =
    D_{\mathrm{GKL}}(\mathbf r\|\mathbf u)
    +
    \sum_{i=1}^{N} r_i\,
    \mathrm{KL}\!\left(q(\cdot\mid i)\|p(\cdot\mid i)\right).
    \label{eq:gkl_chain_main}
  \end{equation}
  The analogous decomposition holds for column marginals and column
  conditionals.
\end{proposition}

\emph{Proof sketch.} Substitute $Q_{ij}=r_iq(j\mid i)$ and
$P_{ij}=u_ip(j\mid i)$ into the definition of GKL. Splitting
$\log(Q_{ij}/P_{ij})$ into a marginal log ratio and a conditional log ratio,
and using that each conditional sums to one, yields
Eq.~\eqref{eq:gkl_chain_main}. The complete proof and extended value
conventions are given in \appref{app:pxrd}.

\begin{corollary}[Deterministic paired targets]
  \label{cor:conditional_ce}
  Let $\pi(i)$ be the target column for row $i$ and let
  $Q_{ij}=r_i\mathbb{I}\{j=\pi(i)\}$. Then
  \[
    D_{\mathrm{GKL}}(\mathbf Q\|\mathbf P)
    =
    D_{\mathrm{GKL}}(\mathbf r\|\mathbf u)
    -\sum_i r_i\log p(\pi(i)\mid i).
  \]
  For the identity target $\tilde{\mathbf P}=N^{-1}\mathbf I_N$, the second
  term is the usual conditional CE with hard labels. If the predicted row marginal
  is fixed to $\mathbf u=N^{-1}\mathbf 1_N$, conditional CE and GKL at the
  coupling level are exactly equal. Under finite UOT marginal penalties, conditional CE
  leaves the marginal GKL term unsupervised.
\end{corollary}

For any positive row-scaling vector $\mathbf c$, conditional CE remains
unchanged because row normalization removes the scaling factors:
\[
\mathcal L_{\mathrm{CE}}(\operatorname{diag}(\mathbf c)\mathbf P)
=\mathcal L_{\mathrm{CE}}(\mathbf P).
\]
It therefore cannot distinguish couplings with the
same row conditionals but different marginals. GKL does distinguish them
because its marginal term changes under row scaling. The objectives agree when
the marginals are fixed; when UOT allows them to vary, GKL also supervises the
marginals. The same argument applies to entropic-OT soft targets.

\paragraph{Training objective.}
PXtal-UOT$_{\mathrm{GKL}}$ directly applies GKL between the target and
predicted transport measures at the coupling level:
$\mathcal L_{\mathrm{GKL}}(\Theta)
=D_{\mathrm{GKL}}(\tilde{\mathbf P}\|
\mathbf P_\Theta^{\mathrm{UOT}})$.
PXtal-UOT$_{\mathrm{CE}}$ uses the same inner UOT solver; the two variants
differ only in whether they supervise the full measure or its normalized
conditionals.
\appref{app:ot} provides pseudocode for this objective and conditional CE.

\section{Experimental Setup}\label{sec:experiments}

\begin{table*}[!t]
  \centering
  {\small
  \setlength{\tabcolsep}{1mm}
  \begin{tabular*}{\textwidth}{@{\extracolsep{\fill}}ll*{7}{c}@{}}
    \toprule
    \multirow{2}{*}{\textbf{Dataset}} & \multirow{2}{*}{\textbf{Model}}
      & \multicolumn{4}{c}{\textbf{Exact retrieval}}
      & \textbf{Crystal system}
      & \multicolumn{2}{c}{\textbf{Space group}} \\
    \cmidrule(lr){3-6}\cmidrule(lr){7-7}\cmidrule(lr){8-9}
    & & Acc@1 (\%) & Acc@3 (\%) & Acc@5 (\%) & MRR
      & Acc@1 (\%) & Acc@1 (\%) & Acc@3 (\%) \\
    \midrule
    \multirow{6}{*}{MP-100}
      & XtalNet-CPCP
      & 65.13 & 84.46 & 89.07 & 75.78 & 87.09 & 81.48 & 90.36 \\
      & PXRDGen-PXE$^{\dagger}$
      & 65.10 & 82.12 & 85.96 & 74.68 & 84.94 & 79.11 & 88.10 \\
      & PXtal-CLIP
      & 67.26 & 84.67 & 88.34 & 76.77 & 87.12 & 81.41 & 90.21 \\
      & PXtal-OT
      & \underline{71.71} & \underline{87.91} & \underline{91.05}
      & \underline{80.43} & 89.05 & \underline{84.46} & \underline{92.13} \\
      & PXtal-UOT$_{\mathrm{CE}}$
      & 70.82 & 87.28 & 90.46 & 79.71 & \underline{89.34} & 84.15 & 91.76 \\
      & \textbf{PXtal-UOT$_{\mathrm{GKL}}$}
      & \textbf{74.45} & \textbf{89.97} & \textbf{93.34}
      & \textbf{82.87} & \textbf{90.57} & \textbf{86.73} & \textbf{93.57} \\
    \midrule
    \multirow{5}{*}{CrystaLLM-v1}
      & PXtal-CLIP
      & 97.73 & 98.70 & 98.95 & 98.29 & 99.20 & 99.13 & 99.72 \\
      & PXtal-OT
      & 98.46 & 99.25 & 99.39 & 98.90 & 99.52 & 99.47 & \underline{99.88} \\
      & PXtal-FGOT
      & 98.29 & 99.22 & 99.35 & 98.79 & 99.52 & 99.48 & 99.86 \\
      & PXtal-UFGOT
      & \underline{98.54} & \underline{99.42} & \underline{99.56}
      & \underline{99.03} & \underline{99.55} & \underline{99.49} & 99.85 \\
      & \textbf{PXtal-UOT$_{\mathrm{GKL}}$}
      & \textbf{98.82} & \textbf{99.53} & \textbf{99.66}
      & \textbf{99.22} & \textbf{99.57} & \textbf{99.54} & \textbf{99.89} \\
    \bottomrule
  \end{tabular*}
  }
  \caption{PXRD-to-crystal candidate retrieval on MP-100 and CrystaLLM-v1 test
  sets. MP-100 uses CSPNet except
  PXRDGen-PXE$^{\dagger}$, which uses CellNet; CrystaLLM-v1 uses GemNet-T.
  Bold and underline indicate the best and second-best models within each
  dataset, respectively.}
  \label{tab:retrieval_test_sets}
\end{table*}

\begin{table*}[!t]
  \centering
  {\small
  \setlength{\tabcolsep}{1mm}
  \begin{tabular*}{\textwidth}{@{\extracolsep{\fill}}l*{12}{c}@{}}
    \toprule
    \multirow{2}{*}{\textbf{Model}}
      & \multicolumn{3}{c}{JARVIS-DFT-3D}
      & \multicolumn{3}{c}{Alex-MP-20}
      & \multicolumn{3}{c}{QMOF}
      & \multicolumn{3}{c}{Perov-5} \\
    \cmidrule(lr){2-4}\cmidrule(lr){5-7}\cmidrule(lr){8-10}\cmidrule(lr){11-13}
      & \shortstack{\textbf{Exact}\\\textbf{retrieval}}
      & \shortstack{\textbf{Crystal}\\\textbf{system}}
      & \shortstack{\textbf{Space}\\\textbf{group}}
      & \shortstack{\textbf{Exact}\\\textbf{retrieval}}
      & \shortstack{\textbf{Crystal}\\\textbf{system}}
      & \shortstack{\textbf{Space}\\\textbf{group}}
      & \shortstack{\textbf{Exact}\\\textbf{retrieval}}
      & \shortstack{\textbf{Crystal}\\\textbf{system}}
      & \shortstack{\textbf{Space}\\\textbf{group}}
      & \shortstack{\textbf{Exact}\\\textbf{retrieval}}
      & \shortstack{\textbf{Crystal}\\\textbf{system}}
      & \shortstack{\textbf{Space}\\\textbf{group}} \\
    \midrule
    PXtal-CLIP
      & 61.66 & 82.16 & 72.30 & 54.48 & 78.83 & 70.48
      & 11.09 & 53.12 & 36.33 & 19.31 & 57.99 & 57.99 \\
    PXtal-OT
      & 64.98 & 84.93 & 76.42 & 56.87 & 80.62 & 72.41
      & 16.32 & 51.67 & 37.87 & 20.18 & 60.05 & 60.03 \\
    PXtal-FGOT
      & 65.36 & 84.33 & 75.99 & 56.39 & 80.07 & 71.98
      & \underline{19.46} & \underline{55.30} & \textbf{42.76}
      & 21.74 & 59.34 & 59.34 \\
    PXtal-UFGOT
      & \underline{68.89} & \underline{89.53} & \underline{81.05}
      & \underline{58.51} & \textbf{83.99} & \underline{75.37}
      & 18.56 & 54.64 & 41.75
      & \underline{23.46} & \underline{62.99} & \underline{62.96} \\
    \textbf{PXtal-UOT$_{\mathrm{GKL}}$}
      & \textbf{70.18} & \textbf{90.22} & \textbf{82.18}
      & \textbf{59.24} & \underline{83.94} & \textbf{75.91}
      & \textbf{20.99} & \textbf{55.87} & \underline{42.54}
      & \textbf{25.63} & \textbf{64.23} & \textbf{64.23} \\
    \bottomrule
  \end{tabular*}
  }
  \caption{Zero-shot PXRD-to-crystal retrieval on four external datasets without
  target-set training or fine tuning. All models use GemNet-T checkpoints
  pretrained on CrystaLLM-v1. The columns report paired-crystal,
  crystal-system, and space-group Acc@1 (\%). Bold and underline indicate the
  best and second-best models; complete metrics are in \appref{app:transfer}.}
  \label{tab:retrieval_transfer}
\end{table*}

\paragraph{Datasets.}
We train on \textsc{MP-100}, a curated Materials Project benchmark
\citep{jain2013materialsproject}, and \textsc{CrystaLLM-v1}, which provides
2,285,702 crystal structures \citep{antunes2024crystal}; each structure is
paired with its simulated PXRD pattern. The four transfer sets cover distinct
classes of crystalline materials. JARVIS-DFT-3D contains a broad set of
three-dimensional inorganic crystals \citep{choudhary2020jarvis}.
Alex-MP-20 contains stable inorganic structures from the Materials Project and
Alexandria with at most 20 atoms per unit cell
\citep{schmidt2022alexandria,zeni2025mattergen}. QMOF comprises
metal--organic frameworks and coordination polymers
\citep{rosen2021qmof,rosen2022qmof}, whereas Perov-5 contains five-atom cubic
$ABX_3$ perovskites \citep{xie2022cdvae,castelli2012new}. Before evaluation,
we removed all crystal structures shared by these transfer sets and the
CrystaLLM-v1 pretraining split. Dataset construction, splits, and PXRD
simulation settings are given in \appref{app:repro}.

\paragraph{Models and optimization.}
MP-100 pretraining pairs a PXRD Transformer over dense patches
\citep{vaswani2017attention} with the CSPNet crystal encoder from DiffCSP
\citep{jiao2023crystal}.
Large-scale CrystaLLM-v1 pretraining pairs the same PXRD encoder with the
MatterGen implementation of GemNet-T
\citep{gasteiger2021gemnet,zeni2025mattergen}, using a global batch size of
8192. Table~\ref{tab:runtime} reports matched pretraining time across
objectives. \appref{app:repro} gives architecture and optimization details;
\appref{app:ot} gives the Optimal Transport solver settings.
\appref{app:retrieval} reports MP-100 retrieval after replacing the default
CSPNet with Matformer \citep{yan2022periodic} or GemNet-T.

\paragraph{Tasks and evaluation protocols.}
We evaluate PXRD-to-crystal candidate retrieval and downstream transfer of
each pretrained encoder. Retrieval ranks a fixed crystal library by
cross-modal cosine similarity. The crystal encoder is evaluated on five
Matbench property tasks; the PXRD encoder is evaluated on symmetry and
unit-cell prediction with probes and fine tuning.

\paragraph{Baselines and ablations.}
External retrieval baselines are XtalNet-CPCP \citep{lai2025end} and
PXRDGen-PXE \citep{li2025powder}. We evaluate both on MP-100 using their
published representation-learning setups; PXRDGen-PXE uses its published
CellNet encoder. Neither provides a CrystaLLM-v1 checkpoint. Comparisons on
CrystaLLM-v1 instead use the same GemNet-T architecture, training data, and
optimization protocol across PXtal objectives.
PXtal-CLIP is the matched InfoNCE baseline. Ablations include balanced
PXtal-OT and PXtal-FGOT and PXtal-UOT$_{\mathrm{CE}}$ with conditional CE.

\begin{figure*}[!t]
  \centering
  \includegraphics[width=\textwidth]{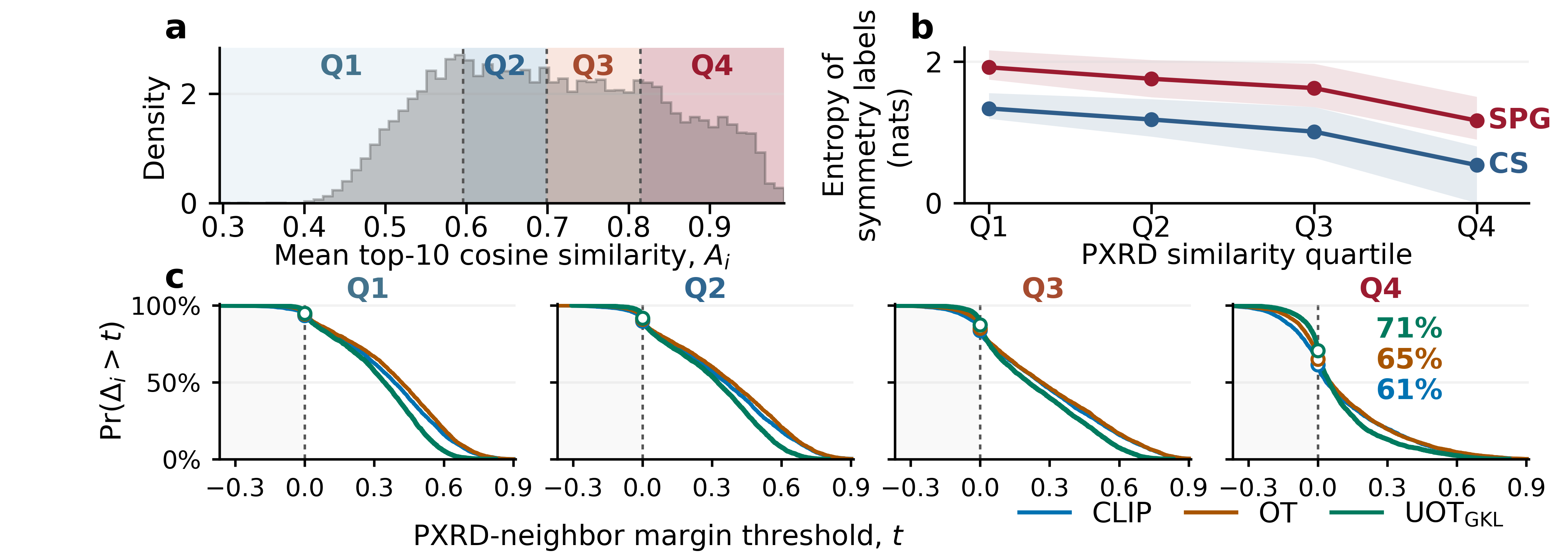}
  \caption{PXRD neighborhood similarity and local ranking on JARVIS-DFT-3D.
  (a) The neighborhood score $A_i$ averages cosine similarity to the ten
  nearest nonpaired PXRD patterns; Q1--Q4 run from lowest to highest similarity.
  (b) Space group (SPG) and crystal system (CS) label entropy across the same
  neighborhoods.
  (c) Survival curves for the PXRD-neighbor margin $\Delta_i$; at $t=0$, a
  positive margin means that the paired crystal outranks every candidate
  associated with the ten nearest nonpaired patterns.}
  \label{fig:information_imbalance_analysis}
\end{figure*}

\begin{table*}[!t]
  \centering
  {\small
  \setlength{\tabcolsep}{1mm}
  \begin{tabular*}{\textwidth}{@{\extracolsep{\fill}}lcccccc@{}}
    \toprule
    \textbf{Model}
      & \shortstack{\textbf{Band gap}\\MAE (eV)$\downarrow$}
      & \shortstack{\textbf{Formation energy}\\MAE (eV/atom)$\downarrow$}
      & \shortstack{\textbf{Metal/nonmetal}\\AUROC$\uparrow$}
      & \shortstack{\textbf{Metal/nonmetal}\\F1$\uparrow$}
      & \shortstack{\textbf{Bulk modulus}\\MAE $\log_{10}(\mathrm{GPa})\downarrow$}
      & \shortstack{\textbf{Shear modulus}\\MAE $\log_{10}(\mathrm{GPa})\downarrow$} \\
    \midrule
    ALIGNN
      & 0.1861 & 0.0215 & 0.9128 & 0.9015 & 0.0568 & 0.0715 \\
    coGN
      & 0.1559 & 0.0170 & 0.9124 & 0.9012 & 0.0535 & 0.0689 \\
    \midrule
    PXtal-CLIP
      & 0.1911 & 0.0231 & \underline{0.9045} & \underline{0.8924}
      & 0.0545 & 0.0729 \\
    PXtal-OT
      & \underline{0.1882} & 0.0229 & \textbf{0.9054} & \textbf{0.8933}
      & \textbf{0.0530} & 0.0720 \\
    PXtal-FGOT
      & \textbf{0.1881} & \textbf{0.0228} & 0.9039 & 0.8914
      & \underline{0.0532} & \underline{0.0716} \\
    PXtal-UFGOT
      & 0.1885 & \underline{0.0228} & 0.9040 & 0.8917
      & 0.0539 & 0.0727 \\
    \textbf{PXtal-UOT$_{\mathrm{GKL}}$}
      & 0.1885 & 0.0228 & 0.9036 & 0.8913
      & 0.0538 & \textbf{0.0715} \\
    \midrule
    From scratch
      & 0.2155 & 0.0245 & 0.8969 & 0.8837 & 0.0643 & 0.0829 \\
    \bottomrule
  \end{tabular*}
  }
  \caption{Crystal encoder transfer on five Matbench v0.1 tasks. MAE and
  AUROC use the official folds and scoring; F1 is also reported. PXtal rows
  fine tune GemNet-T; ALIGNN and coGN are public references. Values are
  five-fold means. Bold and underline mark the best and second-best PXtal
  models from unrounded values; full metrics are in \appref{app:transfer}.}
  \label{tab:matbench_5task_crystal_downstream}
\end{table*}

\section{Results}
\label{sec:results}

\subsection{PXRD-to-Crystal Candidate Retrieval}
\label{subsec:results_retrieval}

\paragraph{Evaluation metrics and candidate libraries.}
We report PXRD-to-crystal retrieval, the nonunique inverse of the
deterministic crystal-to-PXRD map under fixed simulation settings
(Eq.~\ref{eq:pxrd_profile}). Exact Acc@$k$ asks whether the designated crystal
appears among the top $k$ candidates, while mean reciprocal rank (MRR)
summarizes its position in the complete candidate list. Crystal-system and
space-group Acc@$k$ instead ask whether any top-$k$ candidate shares the
corresponding label with the designated crystal.
Table~\ref{tab:retrieval_test_sets} reports PXRD-to-crystal retrieval on MP-100
and CrystaLLM-v1; Table~\ref{tab:retrieval_transfer} reports zero-shot retrieval
on four external datasets.

\paragraph{Retrieval on the pretraining domains.}
Table~\ref{tab:retrieval_test_sets} shows that
PXtal-UOT$_{\mathrm{GKL}}$ leads every reported exact and symmetry-based
retrieval metric on both pretraining-domain test sets. On MP-100, it ranks the
designated crystal first for 74.45\% of PXRD patterns, compared with 67.26\% for
PXtal-CLIP, and raises space-group Acc@1 from 81.41\% to 86.73\%.
PXtal-UOT$_{\mathrm{CE}}$ uses the same unbalanced solver; replacing conditional
CE with coupling-level GKL raises exact Acc@1 from 70.82\% to 74.45\% and MRR
from 79.71 to 82.87 while improving every reported metric.
At million-pair scale on CrystaLLM-v1,
PXtal-UOT$_{\mathrm{GKL}}$ reaches 98.82\% exact Acc@1 and again leads every
reported metric.

\begin{table*}[!t]
  \centering
  {\small
  \setlength{\tabcolsep}{1mm}
  \begin{tabular*}{\textwidth}{@{\extracolsep{\fill}}l*{10}{c}@{}}
    \toprule
    \textbf{Model}
      & \multicolumn{2}{c}{\textbf{Crystal system}}
      & \multicolumn{2}{c}{\textbf{Space group}}
      & \multicolumn{2}{c}{\textbf{Lattice lengths}}
      & \multicolumn{2}{c}{\textbf{Lattice angles}}
      & \multicolumn{2}{c}{\textbf{Unit cell volume}} \\
    \cmidrule(lr){2-3}\cmidrule(lr){4-5}\cmidrule(lr){6-7}
    \cmidrule(lr){8-9}\cmidrule(lr){10-11}
    & Accuracy (\%)$\uparrow$ & F1 (\%)$\uparrow$
    & Accuracy (\%)$\uparrow$ & F1 (\%)$\uparrow$
    & MAE (\AA)$\downarrow$ & $R^2\uparrow$
    & MAE ($^\circ$)$\downarrow$ & $R^2\uparrow$
    & MAE (\AA$^3$)$\downarrow$ & $R^2\uparrow$ \\
    \midrule
    \multicolumn{11}{c}{\textbf{Probe}} \\
    \midrule
    PXtal-CLIP
      & 89.17 & 86.84 & 75.15 & 52.69 & 0.34 & 0.90 & 3.69 & 0.80 & 25.47 & \underline{0.90} \\
    PXtal-OT
      & 89.25 & 86.95 & 75.06 & 52.93 & 0.35 & 0.90 & 3.78 & 0.80 & 25.94 & 0.89 \\
    PXtal-FGOT
      & 88.65 & 86.25 & 74.79 & 52.62 & 0.36 & 0.90 & 3.90 & 0.79 & 26.20 & 0.89 \\
    PXtal-UFGOT
      & \textbf{90.18} & \textbf{88.08} & \textbf{76.78} & \underline{54.00}
      & \textbf{0.30} & \textbf{0.91} & \textbf{3.39} & \underline{0.81} & \underline{23.83} & 0.88 \\
    \textbf{PXtal-UOT$_{\mathrm{GKL}}$}
      & \underline{90.16} & \underline{87.96} & \underline{76.62} & \textbf{54.12}
      & \underline{0.30} & \underline{0.91} & \underline{3.39} & \textbf{0.81} & \textbf{23.48} & \textbf{0.90} \\
    \midrule
    \multicolumn{11}{c}{\textbf{Fine tuning}} \\
    \midrule
    PXtal-CLIP
      & 90.62 & 88.60 & 76.51 & \textbf{53.65} & 0.22 & 0.92 & 2.55 & 0.81 & 20.22 & 0.91 \\
    PXtal-OT
      & 90.84 & 88.76 & 76.43 & 50.94 & 0.21 & 0.93 & \underline{2.36} & 0.82 & 19.19 & \underline{0.91} \\
    PXtal-FGOT
      & 90.52 & 88.50 & 76.44 & 50.78 & 0.21 & 0.92 & \textbf{2.35} & \underline{0.82} & 19.37 & 0.91 \\
    PXtal-UFGOT
      & \underline{91.30} & \underline{89.23} & \textbf{77.74} & \underline{53.18}
      & \underline{0.21} & \textbf{0.93} & 2.44 & 0.82 & \textbf{19.02} & 0.91 \\
    \textbf{PXtal-UOT$_{\mathrm{GKL}}$}
      & \textbf{91.58} & \textbf{89.72} & \underline{77.53} & 52.97
      & \textbf{0.21} & \underline{0.93} & 2.40 & \textbf{0.83} & \underline{19.14} & \textbf{0.91} \\
    \midrule
    From scratch
      & 83.84 & 80.93 & 70.73 & 46.09 & 0.48 & 0.80 & 5.93 & 0.61 & 49.84 & 0.74 \\
    \bottomrule
  \end{tabular*}
  }
  \caption{PXRD encoder transfer on JARVIS-DFT-3D. Probes freeze the
  PXtal-pretrained PXRD Transformer; fine tuning updates its weights. Bold and
  underline mark the best and second-best pretrained models within each setting
  from unrounded values.}
  \label{tab:pxrd_encoder_downstream_jarvis}
\end{table*}

\paragraph{Zero-shot cross-dataset retrieval.}
Without training or fine tuning on the target dataset,
PXtal-UOT$_{\mathrm{GKL}}$ leads PXRD-to-crystal Acc@1 across all four transfer
sets (Table~\ref{tab:retrieval_transfer}). On JARVIS-DFT-3D, it raises Acc@1
from 61.66 to 70.18 and space-group Acc@1 from 72.30 to 82.18 over PXtal-CLIP.
Together with its strong MP-100 result, these gains cover distribution shifts
in Alex-MP-20 and Perov-5 and a clear material-class shift to QMOF.
Absolute accuracy is lower on QMOF and Perov-5: QMOF introduces extended
framework structures, while compositionally distinct Perov-5 crystals share a
common prototype and therefore highly similar reflection families.
PXtal-UOT$_{\mathrm{GKL}}$ still leads both sets, supporting transfer under
these distribution and material-class shifts.

\subsection{PXRD Neighborhood Structure and Local Ranking}
\label{subsec:pxrd_neighborhood_analysis}

\paragraph{PXRD neighborhood similarity.}
PXRD patterns differ in how specifically they constrain a crystal structure.
Before representation learning, for a pattern $x_i^P$, let
$r_{ij}=\mathrm{sim}(x_i^P,x_j^P)$ be the cosine similarity between normalized
intensity profiles and define the neighborhood score as
$A_i=k^{-1}\sum_{j\in\mathcal N_k^P(i)}r_{ij}$, where $k=10$ and
$\mathcal N_k^P(i)$ contains its nearest nonpaired patterns. Higher $A_i$
means a more \emph{crowded} input-space neighborhood. Quartiles run from the lowest
neighborhood similarity in Q1 to the highest in Q4
(Fig.~\ref{fig:information_imbalance_analysis}(a)), without using labels or
learned representations.

Space group (SPG) and crystal system (CS) label entropies decrease from Q1 to
Q4 (Fig.~\ref{fig:information_imbalance_analysis}(b)), so closer PXRD neighbors
increasingly share coarse symmetry labels. They remain crystallographically
distinct: for a pattern in Q4, the ten nearest nonpaired patterns span, on
average, 4.36 SPG and 2.43 CS labels. Even the closest PXRD neighborhoods
therefore contain competing crystal structures, with competition that varies
across patterns. PXtal-UOT addresses this variation through sample-dependent
coupling marginals.

\paragraph{Local ranking within PXRD neighborhoods.}
Having defined the neighborhoods in input space, we test whether the learned
similarity ranks the paired crystal above the crystals associated with its
nearest nonpaired PXRD patterns:
$\Delta_i=s(x_i^P,x_i^C)-\max_{j\in\mathcal N_k^P(i)}s(x_i^P,x_j^C)$, where
$s$ is the learned cosine similarity across modalities. A positive $\Delta_i$
means that the paired crystal outranks every neighboring candidate; otherwise,
a neighboring candidate ranks at least as high.
Figure~\ref{fig:information_imbalance_analysis}(c) plots $\Pr(\Delta_i>t)$.
At $t=0$, positive margins decrease from Q1 to Q4 for all three models, linking
greater PXRD neighborhood similarity to more local ranking reversals. In Q4,
the positive-margin fraction is 61.34\% for PXtal-CLIP, 65.03\% for PXtal-OT,
and 70.74\% for PXtal-UOT$_{\mathrm{GKL}}$. PXtal-UOT$_{\mathrm{GKL}}$
therefore ranks the paired crystal above every neighboring candidate for more
PXRD patterns in Q4, reducing local ranking reversals where the closest PXRD
alternatives compete.

\paragraph{Retrieval over the full candidate library.}
Acc@1 and MRR decrease from Q1 to Q4 for every model. Greater input-space
neighborhood similarity therefore predicts harder retrieval over the complete
candidate library, not only among the ten neighboring candidates.
PXtal-UOT$_{\mathrm{GKL}}$ leads every quartile, and its advantage over
PXtal-CLIP grows toward Q4 (\appref{app:transfer}). In Q4,
PXtal-UOT$_{\mathrm{GKL}}$ reduces local ranking reversals by 9.40 points
and improves full-library Acc@1 by 10.32 points over PXtal-CLIP.
The local and full-library results place its largest gain in crowded PXRD
neighborhoods, where similar patterns create the strongest competition among
distinct crystal structures.

\subsection{Crystal and PXRD Encoder Transfer}
\label{subsec:encoder_transfer}

\paragraph{Crystal encoder transfer.}
Using the five official Matbench v0.1 folds with the
\texttt{matbench} v0.6 package \citep{dunn2020benchmarking}, we
fine tune each PXtal-pretrained GemNet-T initialization on band gap, formation energy,
metal/nonmetal classification, bulk modulus, and shear modulus. With the
backbone and protocol fixed, every pretrained initialization improves all six
metrics over training from scratch
(Table~\ref{tab:matbench_5task_crystal_downstream}), transferring to electronic,
energetic, and elastic properties. ALIGNN
\citep{choudhary2021alignn} and coGN \citep{ruff2024cogn}, included as public
task-specific references, lead on most metrics.

\paragraph{PXRD encoder transfer.}
From PXRD alone, we predict crystal system, space group, and unit-cell geometry.
A probe freezes the encoder; fine tuning updates it. Every pretrained
Transformer outperforms training from scratch
(Table~\ref{tab:pxrd_encoder_downstream_jarvis}). PXtal-UFGOT and
PXtal-UOT$_{\mathrm{GKL}}$ lead the probe results, while
PXtal-UOT$_{\mathrm{GKL}}$ improves most displayed fine-tuning metrics over
PXtal-CLIP. \appref{app:transfer} reports MP-100 probes for all three crystal
encoders.

\section{Conclusion and Future Work}
We introduced PXtal, which treats joint PXRD and crystal representation
learning as cross-modal alignment under physically imposed information
asymmetry. By letting the transport coupling adapt to how strongly each pattern
constrains its structure, PXtal improves PXRD-to-crystal retrieval, especially
for patterns with close but crystallographically distinct neighbors, and yields
encoders that transfer to downstream tasks. Future work will extend this
approach to measured, multiphase, and open-set PXRD. More broadly, our results
suggest that alignment objectives in scientific multimodal learning should
reflect what each modality preserves.

\section*{Impact Statement}
This work aims to advance machine learning for materials
characterization and discovery by improving representation
learning between powder X-ray diffraction patterns and crystal structures. The main positive impact is scientific: better
PXRD and crystal alignment can make database search,
structure retrieval, and downstream materials-property modeling more data-efficient, especially when diffraction evidence is ambiguous. In the long term, such tools may help
researchers navigate large chemical and structural spaces,
reduce trial-and-error in materials screening, and accelerate
the identification of candidate functional materials.

\section*{Acknowledgements}
This work was supported by Johnson Matthey and UK
Engineering and Physical Sciences Research Council (EPSRC) via Digital Navigation of Chemical Space for Function (EP/V026887/1) and AI for Chemistry: AIchemy Hub
(EP/Y028775/1 and EP/Y028759/1). We thank Andrew
York, Tugce Eralp Erden and Edd Bilbe for valuable discussions and industrial mentorship, which helped shape
the motivation and interpretation of this work. V. V. Gusev is supported by the Leverhulme Research Fellowship.
We thank Andrij Vasylenko for supporting this work and
access to the Isambard-AI facility under project AIRR-IN
(BRIDGE, November 6, 2025), which provided the GPU
resources used in this work. The authors acknowledge the use of resources
provided by the Isambard-AI National AI Research Resource (AIRR). Isambard-AI
is operated by the University of Bristol and is funded by the UK Government's
Department for Science, Innovation and Technology (DSIT) via UK Research and
Innovation; and the Science and Technology Facilities Council
[ST/AIRR/I-A-I/1023]. Isambard-AI is described by
\citet{mcintoshsmith2024isambard}.

{\small
\bibliographystyle{preprintnat}
\bibliography{reference/reference}
}

\clearpage
\newgeometry{letterpaper,left=1in,right=1in,top=0.9in,bottom=0.9in,footskip=0.35in}
\onecolumn
\appendix
\phantomsection
\pdfbookmark[0]{Supplementary Material}{supplementary}
\begin{center}
{\LARGE\bfseries Supplementary Material}\\[0.7em]
{\large PXtal: Learning to Align Powder X-Ray Diffraction and Crystal Structures\\
under Information Asymmetry across Modalities}
\end{center}
\vspace{0.6em}
\noindent\rule{\textwidth}{0.4pt}
\vspace{0.2em}

\noindent{\bfseries Contents}\par\vspace{0.4em}
{\renewcommand{\arraystretch}{1.25}
\noindent\begin{tabular*}{\textwidth}{@{}l@{\hspace{1em}}l@{\extracolsep{\fill}}r@{}}
\ref{app:pxrd} & PXRD Formulation and Proofs & \pageref{app:pxrd}\\
\ref{app:retrieval} & Additional Retrieval Results & \pageref{app:retrieval}\\
\ref{app:transfer} & Transfer, Neighborhood, and Encoder Results & \pageref{app:transfer}\\
\ref{app:repro} & Reproducibility Details & \pageref{app:repro}\\
\ref{app:ot} & Optimal Transport Algorithms and Implementation & \pageref{app:ot}\\
\end{tabular*}}
\vspace{0.2em}
\noindent\rule{\textwidth}{0.4pt}
\vspace{0.6em}

\noindent \appref{app:pxrd} expands the PXRD formulation and gives the full proof of the
generalized Kullback--Leibler (GKL) decomposition in the main text.
Appendices~\ref{app:retrieval} and~\ref{app:transfer} contain the additional retrieval and encoder results cited
there. Appendices~\ref{app:repro} and~\ref{app:ot} document the data, models, training, solvers, and
algorithms.

\section{PXRD Formulation and Proofs}\label{app:pxrd}
\subsection{GSAS-II Crystallographic Forward Model}
\label{sec:pxrd_background}

Equation~\eqref{eq:pxrd_profile} in the main text summarizes a single-phase kinematical powder
profile \citep{rietveld1969profile,young1993rietveld}. The definitions below
follow standard crystallographic conventions
\citep{dinnebier2008principles,internationaltables2019powder}. The simulations
use the constant-wavelength X-ray powder model in GSAS-II \citep{toby2013}. Let
$\mathbf h=(h,k,l)^\top\in\mathbb Z^3$ be a
Miller-index vector. The reciprocal basis
$\mathbf a^*,\mathbf b^*,\mathbf c^*$ follows the crystallographic
convention and contains no factor of $2\pi$. Define
$\mathbf d^*_{hkl}=h\mathbf a^*+k\mathbf b^*+l\mathbf c^*$.
Then $d_{hkl}=\lVert\mathbf d^*_{hkl}\rVert^{-1}$, where
$\lVert\cdot\rVert$ is the Euclidean norm. With X-ray wavelength
$\lambda$ and Bragg half-angle $\theta_{hkl}$,
\begin{equation}
  s_{hkl}
  =\frac{\lVert\mathbf d^*_{hkl}\rVert}{2}
  =\frac{\sin\theta_{hkl}}{\lambda},
  \qquad
  2d_{hkl}\sin\theta_{hkl}=\lambda .
  \label{eq:bragg_reciprocal_supp}
\end{equation}
Here $s_{hkl}=1/(2d_{hkl})$ is the crystallographic scattering variable.
The Crystallographic Information File (CIF) supplies asymmetric-unit sites,
and GSAS-II applies the space-group operations. An equivalent
symmetry-expanded unit-cell expression for the complex structure factor is
\begin{equation}
  F_{hkl}
  =
  \sum_{j=1}^{N_{\mathrm{uc}}}
  o_j f_j(s_{hkl},\lambda)T_j(hkl)
  \exp\!\left[2\pi\mathrm{i}
  (hx_j+ky_j+lz_j)\right].
  \label{eq:structure_factor_supp}
\end{equation}
$N_{\mathrm{uc}}$ counts the atomic site contributions in one
symmetry-expanded unit cell; partially occupied or compositionally disordered
sites contribute with their own occupancies. The index $j$ identifies one
such contribution, $(x_j,y_j,z_j)$ gives its fractional coordinates,
$o_j$ is its site occupancy, $f_j$ is its complex X-ray scattering factor,
and $T_j(hkl)$ is its atomic displacement (Debye--Waller) factor. Tensors on
symmetry-generated sites are transformed with the corresponding space-group
operation. The symbol $\mathrm i$ is the imaginary unit. GSAS-II writes the
atomic scattering factor as
\begin{equation}
  f_j(s,\lambda)
  =
  f_j^0(s)+\Delta f'_j(\lambda)
  +\mathrm{i}\Delta f''_j(\lambda)
  \label{eq:atomic_form_factor_supp}
\end{equation}
with $s=\sin\theta/\lambda$ in \AA$^{-1}$. The normal term is
$f_j^0(s)=\sum_{n=1}^{4}a_{jn}\exp(-b_{jn}s^2)+c_j$.
GSAS-II takes the Cromer--Waber coefficients from its X-ray form-factor
table and evaluates the real and imaginary anomalous terms,
$\Delta f'_j$ and $\Delta f''_j$, with the Cromer--Liberman calculation
\citep{cromer1974atomic,cromer1981anomalous,toby2013}.
The coefficients $a_{jn}$ and $c_j$ use electron scattering units, and
$b_{jn}$ has units of \AA$^2$.
For the isotropic atomic displacement parameter $U_{\mathrm{iso},j}$,
\begin{equation}
  T_j^{\mathrm{iso}}(hkl)
  =\exp(-B_j s_{hkl}^2),
  \qquad B_j=8\pi^2U_{\mathrm{iso},j}.
  \label{eq:isotropic_dw_supp}
\end{equation}
$B_j$ is the corresponding isotropic $B$ parameter; both $B_j$ and
$U_{\mathrm{iso},j}$ have units of \AA$^2$. With
$a^*=\lVert\mathbf a^*\rVert$ and analogous definitions for $b^*$ and
$c^*$, the anisotropic factor is
\begin{equation}
  \begin{aligned}
  T_j^{\mathrm{aniso}}(hkl)
  &=\exp\!\Bigl\{-2\pi^2\bigl[
  h^2a^{*2}U_{11,j}+k^2b^{*2}U_{22,j}
  +l^2c^{*2}U_{33,j}\\
  &\qquad{}+2hka^*b^*U_{12,j}
  +2hla^*c^*U_{13,j}
  +2klb^*c^*U_{23,j}\bigr]\Bigr\}.
  \end{aligned}
  \label{eq:anisotropic_dw_supp}
\end{equation}
The six $U_{pq,j}$ components refer to the crystallographic basis and have
units of \AA$^2$. Thus $T_j(hkl)=T_j^{\mathrm{iso}}(hkl)$ for a site with
an isotropic displacement parameter, and
$T_j(hkl)=T_j^{\mathrm{aniso}}(hkl)$ for a site with anisotropic displacement
parameters. Equations~\eqref{eq:structure_factor_supp}--
\eqref{eq:anisotropic_dw_supp} follow the standard crystallographic
definitions
\citep{internationaltables2019powder,dinnebier2023intensity}. The integrated
kinematical intensity contains
$F_{hkl}F_{hkl}^*=|F_{hkl}|^2$; the asterisk denotes complex conjugation.

Let $\mathcal R$ be the allowed powder reflection list generated by GSAS-II.
Each record $\kappa\in\mathcal R$ stores a representative
$hkl$, its multiplicity $m_\kappa$, spacing $d_\kappa$, and calculated
$|F_\kappa|^2$. The constant-wavelength X-ray profile is
\begin{equation}
  \begin{aligned}
  y_{\mathrm{calc}}(2\theta_i)
    &=b(2\theta_i)+\sum_{\kappa\in\mathcal R}
      I_\kappa\Phi_\kappa(2\theta_i-2\theta_\kappa),\\
  I_\kappa
    &=S_hS_pV^*m_\kappa L_\kappa\Pi_\kappa
      O_\kappa A_\kappa E_\kappa|F_\kappa|^2 .
  \end{aligned}
  \label{eq:pxrd_full_supp}
\end{equation}
Here $i$ indexes the sampled $2\theta$ grid. The factors $S_h$ and $S_p$ are
the histogram and phase scales, $V^*=\sqrt{\det\mathbf G}=1/V$ is the
reciprocal unit-cell volume, and $\mathbf G$ is the reciprocal metric tensor.
$L_\kappa$, $\Pi_\kappa$, $O_\kappa$, $A_\kappa$, and $E_\kappa$ are the
Lorentz, polarization, preferred-orientation, absorption, and extinction
corrections. The function $\Phi_\kappa$ is the sampled peak-profile function,
centered at $2\theta_\kappa$, and $b$ is the background. Its subscript also
allows the width and asymmetry to vary by reflection. This is the GSAS-II
constant-wavelength X-ray (PXC) intensity expression
\citep{rietveld1969profile,young1993rietveld,dinnebier2023intensity}.
For the present instrument, $\mathrm{Polariz.}=0.5$ and
$\mathrm{Azimuth}=0$, so
\[
  L_\kappa=\frac{1}{2\sin^2\vartheta_\kappa\cos\vartheta_\kappa},
  \qquad
  \Pi_\kappa=\frac{1+\cos^2(2\vartheta_\kappa)}{2},
\]
where $\vartheta_\kappa$ is half the final peak-center angle. The simulation
sets $S_h=S_p=1$, $O_\kappa=A_\kappa=E_\kappa=1$, and
$b(2\theta_i)=1$. Its actual profile therefore reduces to
\[
  y_{\mathrm{calc}}(2\theta_i)
  =1+V^*\sum_{\kappa\in\mathcal R}
    m_\kappa L_\kappa\Pi_\kappa|F_\kappa|^2
    \Phi_\kappa(2\theta_i-2\theta_\kappa).
\]
Equation~\eqref{eq:pxrd_profile} in the main text uses $\phi$ for this peak profile and absorbs
the scale, volume, multiplicity, and correction factors into $w_{hkl}$. Its
$\approx$ sign marks this compact reflection-level notation.

\begin{table*}[!htbp]
  \centering
  \scriptsize
  \setlength{\tabcolsep}{4.0mm}
  \captionsetup{justification=raggedright,singlelinecheck=false}
  \begin{tabular}{@{}llll@{}}
    \toprule
    \textbf{Instrument field} & \textbf{Value}
      & \textbf{Instrument field} & \textbf{Value} \\
    \midrule
    Type & PXC & Source label & CuKa \\
    Lam & $1.5403$\,\AA{} & Zero
      & $0.049948848393250284^\circ$ \\
    Polariz. & $0.5$ & Azimuth & $0.0^\circ$ \\
    U & $418.99266761626023$ & V & $-227.1108015369226$ \\
    W & $48.813562400057585$ & X & $0.0$ \\
    Y & $15.737032230664841$ & Z & $0.0$ \\
    SH/L & $0.027292127957733888$ & Bank & $1$ \\
    Diff-type & Debye--Scherrer & Goniometer radius & $200.0$\,mm \\
    \bottomrule
  \end{tabular}
  \caption{GSAS-II constant-wavelength instrument parameters used for all
    simulations. \texttt{CuKa} labels the source; the single \texttt{Lam}
    value defines one wavelength, not a resolved
    $K\alpha_1/K\alpha_2$ doublet.}
  \label{tab:gsasii_instrument}
\end{table*}

\subsection{GSAS-II Simulation Protocol}
\label{sec:exact_gsasii_simulation}

For each CIF, the simulation imports one phase and creates one GSAS-II powder
histogram. Table~\ref{tab:gsasii_instrument} gives every
instrument parameter. \texttt{PXC} denotes constant-wavelength X-ray powder
data; \texttt{Lam} gives the wavelength; and \texttt{Zero} gives a constant
$2\theta$ offset. \texttt{Polariz.} and \texttt{Azimuth} define the
polarization factor. The coefficients \texttt{U}, \texttt{V}, and \texttt{W}
set Gaussian instrumental broadening, while \texttt{X}, \texttt{Y}, and
\texttt{Z} set Lorentzian broadening. \texttt{SH/L} is the
$(S+H)/L$ axial-divergence parameter in the
Finger--Cox--Jephcoat profile \citep{finger1994asymmetry}. The goniometer
radius enters the sample
displacement correction, which is zero here.

For reflection $\kappa$, define the ideal Bragg angle
$\theta_\kappa^{B}=\arcsin[\lambda/(2d_\kappa)]$. With zero sample
displacement and zero hydrostatic strain, GSAS-II places the peak at
\begin{equation}
  2\theta_\kappa
  =
  2\theta_\kappa^{B}
  +\mathrm{Zero}.
  \label{eq:gsasii_peak_position}
\end{equation}
Set $\vartheta_\kappa=\theta_\kappa^{B}+\mathrm{Zero}/2$, the half-angle of
this final center.
With the default Lorentzian size and microstrain mixing used here, the
GSAS-II Gaussian variance parameter and Lorentzian full width at half maximum
are
\begin{equation}
  \begin{aligned}
  \sigma_{\kappa}^{2}
    &=U\tan^2\vartheta_\kappa+V\tan\vartheta_\kappa+W,\\
  \gamma_{\kappa}
    &=\frac{X}{\cos\vartheta_\kappa}
      +Y\tan\vartheta_\kappa+Z
      +\gamma_{\mathrm{size},\kappa}
      +\gamma_{\mathrm{strain},\kappa},\\
  \gamma_{\mathrm{size},\kappa}
    &=\frac{1.8\lambda}
      {\pi D_{\mathrm{size}}\cos\vartheta_\kappa},\qquad
  \gamma_{\mathrm{strain},\kappa}
    =\frac{0.018e_\mu\tan\vartheta_\kappa}{\pi}.
  \end{aligned}
  \label{eq:gsasii_profile_widths}
\end{equation}
GSAS-II stores $\sigma_\kappa^2$ in centidegree-squared and
$\gamma_\kappa$ in centidegrees. Thus $U,V,W$ use
centidegree-squared units, while $X,Y,Z$ use centidegrees.
$D_{\mathrm{size}}=1\,\mu$m is the isotropic crystallite size, and
$e_\mu=1000$ is the isotropic microstrain in units of $10^{-6}$.
The numerical constants in Eq.~\eqref{eq:gsasii_profile_widths} return
centidegrees when $\lambda$ is in \AA{} and $D_{\mathrm{size}}$ is in
$\mu$m, as in GSAS-II.
The program floors both stored width parameters at $0.001$; the values in
Table~\ref{tab:gsasii_instrument} remain above that floor on both angular
grids. It then evaluates the Finger--Cox--Jephcoat axial-divergence-modified
Thompson--Cox--Hastings pseudo-Voigt profile
\citep{caglioti1958collimators,thompson1987profile,
finger1994asymmetry,dinnebier2021profiles}.

The fixed calculation uses these GSAS-II values \citep{toby2013}:
\begin{itemize}
  \item one phase with phase and histogram scales equal to 1;
  \item a three-term \emph{chebyschev-1} background with
    coefficients $(1,0,0)$;
  \item Debye--Scherrer geometry, cylinder absorption parameter $\mu R=0$,
    zero sample displacement, zero extinction, zero hydrostatic strain, and
    zero Babinet terms;
  \item March--Dollase ratio 1 for random orientation
    \citep{dollase1986orientation};
  \item isotropic crystallite size $1\,\mu$m and microstrain
    $1000\times10^{-6}$, both with Lorentzian mixing coefficient 1.
\end{itemize}
The CIF supplies the unit cell, space group, atom types, coordinates,
occupancies, and displacement parameters. GSAS-II generates the allowed
reflections and multiplicities. A missing occupancy defaults to $o_j=1$; a
missing isotropic displacement parameter defaults to
$U_{\mathrm{iso},j}=0.01$\,\AA$^2$. The importer converts
$B_{\mathrm{iso}}$ through $U_{\mathrm{iso}}=B_{\mathrm{iso}}/(8\pi^2)$
and uses anisotropic $U_{pq}$ values when present.

An empty refinement step evaluates the profile with zero refinement cycles,
so GSAS-II does not optimize any parameter. The simulation reads
\texttt{ycalc}, takes its absolute value, and discards the simulated observed
channel. The resulting profile is deterministic and contains no added
counting noise.

MP-100 uses a 2001-point source grid over
$2\theta\in[5^\circ,65^\circ]$, including both endpoints, with
$0.03^\circ$ spacing. The data loader linearly interpolates that curve to a
2048-point model grid, also including both endpoints, before normalization.
CrystaLLM-v1 and the four external transfer datasets use 4096 points over
$[5^\circ,90^\circ]$, including both endpoints, with no length change.
For each pattern, set $z_i=|y_{\mathrm{calc}}(2\theta_i)|$. The loader applies
\begin{equation}
  \widehat y_i
  =
  \frac{z_i-\min_k z_k}
       {\max_k z_k-\min_k z_k},
  \label{eq:pxrd_minmax_supp}
\end{equation}
where $i$ indexes a model-grid point and $k$ ranges over the same pattern.
The loader clips $\widehat y_i$ to $[0,1]$; a constant pattern maps to zero.

\FloatBarrier
\subsection{Formal Proof of \texorpdfstring{Proposition~\ref{thm:gkl_decomposition}}{Proposition 3.1}}
\label{app:proofs}

For $x,y\geq0$, define the extended scalar GKL term
\begin{equation}
  g(x,y)=x\log(x/y)-x+y,
  \label{eq:scalar_gkl_supp}
\end{equation}
with $g(0,0)=0$, $g(0,y)=y$, and $g(x,0)=+\infty$ for $x>0$.
The logarithm is natural. For arrays of the same shape,
$D_{\mathrm{GKL}}(\mathbf A\|\mathbf B)$ is the sum of $g(A_{ij},B_{ij})$
over every entry; for vectors, it is the corresponding one-index sum.
All KL sums use the standard lower-semicontinuous convention
$0\log(0/y)=0$, including $y=0$. For probability vectors $\mathbf q$ and
$\mathbf p$,
$\mathrm{KL}(\mathbf q\|\mathbf p)=\sum_jq_j\log(q_j/p_j)$ under the same
support convention.

\begin{propositionrestated}
Let $\mathbf Q,\mathbf P\in\mathbb R_+^{N\times M}$ have strictly positive
row sums $\mathbf r=\mathbf Q\mathbf1_M$ and
$\mathbf u=\mathbf P\mathbf1_M$. Assume $Q_{ij}>0$ implies $P_{ij}>0$, and
define the row-conditional distributions
\[
  q(j\mid i)=\frac{Q_{ij}}{r_i},
  \qquad
  p(j\mid i)=\frac{P_{ij}}{u_i}.
\]
Here $i\in\{1,\ldots,N\}$ indexes rows,
$j\in\{1,\ldots,M\}$ indexes columns, and $\mathbf1_n$ is the length-$n$
all-ones vector. The set $\mathbb R_+$ contains the nonnegative real
numbers. Then
\begin{equation}
  D_{\mathrm{GKL}}(\mathbf Q\|\mathbf P)
  =D_{\mathrm{GKL}}(\mathbf r\|\mathbf u)
  +\sum_{i=1}^{N}r_i\,
  \mathrm{KL}\!\left(q(\cdot\mid i)\|p(\cdot\mid i)\right).
  \label{eq:gkl_chain_supp}
\end{equation}
If the column sums
$\mathbf c=\mathbf Q^\top\mathbf1_N$ and
$\mathbf v=\mathbf P^\top\mathbf1_N$ are also strictly positive, the analogous
identity holds for column marginals and column conditionals.
\end{propositionrestated}

\begin{proof}
Write $q_{ij}=q(j\mid i)$ and $p_{ij}=p(j\mid i)$. Then
$Q_{ij}=r_iq_{ij}$, $P_{ij}=u_ip_{ij}$, and both conditionals sum to one.
Set
$\mathcal K_i=\mathrm{KL}(q(\cdot\mid i)\|p(\cdot\mid i))$.
The support assumption makes each term finite on the support of $\mathbf Q$.
Direct expansion gives
\begingroup
\allowdisplaybreaks[4]
\begin{align}
  D_{\mathrm{GKL}}(\mathbf Q\|\mathbf P)
  &=
  \sum_{i,j}
  \left[
  r_iq_{ij}\log\frac{r_iq_{ij}}{u_ip_{ij}}
  -r_iq_{ij}+u_ip_{ij}
  \right]
  \nonumber\\
  &=
  \sum_{i,j}r_iq_{ij}\log\frac{r_i}{u_i}
  +\sum_{i,j}r_iq_{ij}\log\frac{q_{ij}}{p_{ij}}
  -\sum_i r_i\sum_jq_{ij}
  +\sum_i u_i\sum_jp_{ij}
  \nonumber\\
  &=
  \sum_i\left[
  r_i\log\frac{r_i}{u_i}-r_i+u_i
  \right]
  +\sum_i r_i\sum_jq_{ij}\log\frac{q_{ij}}{p_{ij}}
  \nonumber\\
  &=
  D_{\mathrm{GKL}}(\mathbf r\|\mathbf u)
  +\sum_i r_i\mathcal K_i.
\end{align}
\endgroup
This proves Eq.~\eqref{eq:gkl_chain_supp}. Applying the same row argument to
$\mathbf Q^\top$ and $\mathbf P^\top$, with
$\bar q(i\mid j)=Q_{ij}/c_j$ and
$\bar p(i\mid j)=P_{ij}/v_j$, yields
\begin{equation}
  D_{\mathrm{GKL}}(\mathbf Q\|\mathbf P)
  =D_{\mathrm{GKL}}(\mathbf c\|\mathbf v)
  +\!\sum_{j=1}^{M}c_j
  \mathrm{KL}\!\left(\bar q(\cdot\mid j)\|\bar p(\cdot\mid j)\right).
  \label{eq:gkl_chain_columns_supp}
\end{equation}
Here $\bar q(\cdot\mid j)$ and $\bar p(\cdot\mid j)$ are the target and
predicted row distributions conditional on column $j$.
\end{proof}

\begin{corollaryrestated}
Let $\pi(i)$ be the target column for row $i$ and
$Q_{ij}=r_i\mathbb{I}\{j=\pi(i)\}$, where $\mathbb{I}\{\cdot\}$ is the
indicator function. Then
\begin{equation}
  D_{\mathrm{GKL}}(\mathbf Q\|\mathbf P)
  =
  D_{\mathrm{GKL}}(\mathbf r\|\mathbf u)
  -\sum_i r_i\log p(\pi(i)\mid i).
  \label{eq:gkl_deterministic_supp}
\end{equation}
In the square paired case $M=N$, $\mathbf I_N$ denotes the $N\times N$
identity matrix. For $\mathbf Q=N^{-1}\mathbf I_N$, the second term is
hard-label row conditional cross entropy. If
$\mathbf u=\mathbf r=N^{-1}\mathbf1_N$, that cross entropy equals coupling-level
GKL.
\end{corollaryrestated}

\begin{proof}
The target conditional $q(\cdot\mid i)$ is a point mass at $\pi(i)$, so
\[
  \mathrm{KL}\!\left(q(\cdot\mid i)\|p(\cdot\mid i)\right)
  =-\log p(\pi(i)\mid i).
\]
Substitution into Eq.~\eqref{eq:gkl_chain_supp} proves the result.
\end{proof}

\paragraph{Effect of row rescaling.}
For $\boldsymbol\eta\in\mathbb R_{>0}^{N}$, set
$\mathbf P^{(\boldsymbol\eta)}
=\operatorname{diag}(\boldsymbol\eta)\mathbf P$. Its row marginal is
$\mathbf u^{(\boldsymbol\eta)}
=\boldsymbol\eta\odot\mathbf u$, while
\begin{equation}
  p^{(\boldsymbol\eta)}(j\mid i)
  =
  \frac{\eta_iP_{ij}}{\sum_\ell\eta_iP_{i\ell}}
  =p(j\mid i).
  \label{eq:conditional_scale_invariance}
\end{equation}
The operator $\operatorname{diag}$ places a vector on a diagonal,
$\odot$ denotes elementwise multiplication, and $\ell$ indexes columns in
the row-normalizing sum.
Conditional cross entropy is unchanged by any positive row rescaling. GKL
changes by
\begin{equation}
  D_{\mathrm{GKL}}(\mathbf Q\|
  \mathbf P^{(\boldsymbol\eta)})
  -D_{\mathrm{GKL}}(\mathbf Q\|\mathbf P)
  =
  \sum_i\left[u_i(\eta_i-1)-r_i\log\eta_i\right].
  \label{eq:gkl_row_scaling_difference}
\end{equation}
GKL changes with row mass. For fixed conditionals, its marginal term has the
unique minimizer
$\eta_i=r_i/u_i$. For its $i$th scalar contribution,
\[
  \begin{aligned}
  f_i(\eta_i)
  &=r_i\log\frac{r_i}{\eta_i u_i}-r_i+\eta_i u_i,\\
  f_i'(\eta_i)&=u_i-\frac{r_i}{\eta_i},
  \qquad
  f_i''(\eta_i)=\frac{r_i}{\eta_i^2}>0.
  \end{aligned}
\]
The primes denote derivatives with respect to $\eta_i$.
Strict convexity gives a unique mass-matching minimizer. Two distinct
rescalings can still have the same scalar GKL value.

\paragraph{Soft targets.}
The arrays $q=(q_{ij})$ and $p=(p_{ij})$ collect the row conditionals defined
above. Define their target-marginal-weighted conditional cross entropy and
target entropy by
\begin{align}
  \mathrm{CE}_{\mathbf r}(q,p)
  &=-\sum_i r_i\sum_jq_{ij}\log p_{ij},\\
  \mathcal H_{\mathbf r}(q)
  &=-\sum_i r_i\sum_jq_{ij}\log q_{ij}.
\end{align}
Equation~\eqref{eq:gkl_chain_supp} becomes
\begin{equation}
  D_{\mathrm{GKL}}(\mathbf Q\|\mathbf P)
  =
  D_{\mathrm{GKL}}(\mathbf r\|\mathbf u)
  +\mathrm{CE}_{\mathbf r}(q,p)
  -\mathcal H_{\mathbf r}(q).
  \label{eq:gkl_soft_ce_supp}
\end{equation}
The last term depends only on the fixed target. When
$\mathbf u=\mathbf r$, coupling-level GKL and weighted soft cross entropy
differ by that constant and have the same gradients with respect to
$\mathbf P$. Under UOT, $\mathbf u$ can vary, and GKL also penalizes the
predicted marginal through its first term.

\paragraph{Nonnegativity and equality.}
For $y>0$, write $g(x,y)=y\varphi(x/y)$ with
$t=x/y$ and $\varphi(t)=t\log t-t+1$. Since
$\varphi''(t)=1/t>0$ and $\varphi$ has its unique minimum
$\varphi(1)=0$, every scalar term is nonnegative and vanishes exactly when
$x=y$. Here the primes denote derivatives with respect to $t$. The boundary
values defined above satisfy the same property. It
follows that
\[
  D_{\mathrm{GKL}}(\mathbf Q\|\mathbf P)\geq0,
  \qquad
  D_{\mathrm{GKL}}(\mathbf Q\|\mathbf P)=0
  \Longleftrightarrow
  \mathbf P=\mathbf Q.
\]

\FloatBarrier
\FloatBarrier
\section{Additional Retrieval Results}\label{app:retrieval}
\label{app:retrieval_full}

Table~\ref{tab:mp100_backbone} adds MP-100 retrieval results for Matformer
\citep{yan2022periodic} and GemNet-T \citep{gasteiger2021gemnet} crystal
encoders. The main text gives the corresponding CSPNet results
\citep{jiao2023crystal}.

\begin{table*}[!htbp]
  \centering
  \small
  \setlength{\tabcolsep}{1mm}
  \begin{tabular}{@{}ll*{7}{r}@{}}
    \toprule
    & & \multicolumn{4}{c}{\textbf{Exact retrieval}}
      & \multicolumn{1}{c}{\textbf{Crystal system}}
      & \multicolumn{2}{c}{\textbf{Space group}} \\
    \cmidrule(lr){3-6}\cmidrule(lr){7-7}\cmidrule(lr){8-9}
    \textbf{Encoder} & \textbf{Model}
      & Acc@1$\uparrow$ & Acc@3$\uparrow$ & Acc@5$\uparrow$ & MRR$\uparrow$
      & Acc@1$\uparrow$ & Acc@1$\uparrow$ & Acc@3$\uparrow$ \\
    \midrule
    \multirow{4}{*}{Matformer} & PXtal-CLIP & 73.50 & 86.89 & 89.78 & 80.98 & 86.89 & 82.02 & 90.49 \\
     & PXtal-OT & \underline{76.33} & \underline{89.34} & \underline{91.78} & \underline{83.36} & \underline{88.46} & \underline{83.77} & \underline{92.18} \\
     & PXtal-UOT$_{\mathrm{CE}}$ & 73.13 & 86.28 & 89.13 & 80.60 & 86.41 & 81.51 & 89.89 \\
     & PXtal-UOT$_{\mathrm{GKL}}$ & \textbf{76.79} & \textbf{89.67} & \textbf{92.25} & \textbf{83.84} & \textbf{88.62} & \textbf{84.35} & \textbf{92.39} \\
    \midrule
    \multirow{4}{*}{GemNet-T} & PXtal-CLIP & 96.97 & 99.77 & 99.87 & 98.36 & 99.45 & 99.15 & 99.81 \\
     & PXtal-OT & \underline{97.19} & \textbf{99.92} & \textbf{99.95} & \underline{98.51} & 99.44 & \underline{99.19} & \textbf{99.92} \\
     & PXtal-UOT$_{\mathrm{CE}}$ & 97.10 & 99.77 & 99.85 & 98.44 & \underline{99.52} & 99.16 & 99.79 \\
     & PXtal-UOT$_{\mathrm{GKL}}$ & \textbf{97.44} & \underline{99.87} & \underline{99.92} & \textbf{98.64} & \textbf{99.63} & \textbf{99.27} & \underline{99.90} \\
    \bottomrule
  \end{tabular}
  \caption{Crystal encoder ablation for PXRD-to-crystal retrieval on MP-100.
    Columns give exact Acc@1/3/5 and MRR, crystal-system Acc@1, and space-group
    Acc@1/3, all in percent. Bold and underline mark the best and second-best
    models within each encoder.}
  \label{tab:mp100_backbone}
\end{table*}

\FloatBarrier
\section{Transfer, Neighborhood, and Encoder Results}\label{app:transfer}
\label{app:extended_results}

\begin{table*}[!htbp]
  \centering
  \small
  \setlength{\tabcolsep}{1mm}
  \begin{tabular*}{\textwidth}{@{\extracolsep{\fill}}ll*{7}{r}@{}}
    \toprule
    & & \multicolumn{4}{c}{\textbf{Exact retrieval}}
      & \multicolumn{1}{c}{\textbf{Crystal system}}
      & \multicolumn{2}{c}{\textbf{Space group}} \\
    \cmidrule(lr){3-6}\cmidrule(lr){7-7}\cmidrule(lr){8-9}
    \textbf{Dataset} & \textbf{Model}
      & Acc@1$\uparrow$ & Acc@3$\uparrow$ & Acc@5$\uparrow$ & MRR$\uparrow$
      & Acc@1$\uparrow$ & Acc@1$\uparrow$ & Acc@3$\uparrow$ \\
    \midrule
    \multirow{5}{*}{JARVIS-DFT-3D}
      & PXtal-CLIP
      & 61.66 & 76.27 & 80.34 & 70.09 & 82.16 & 72.30 & 84.03 \\
      & PXtal-OT
      & 64.98 & 79.93 & 83.89 & 73.48 & 84.93 & 76.42 & 87.37 \\
      & PXtal-FGOT
      & 65.36 & 79.90 & 83.68 & 73.61 & 84.33 & 75.99 & 87.28 \\
      & PXtal-UFGOT
      & \underline{68.89} & \underline{84.32} & \underline{88.62}
      & \underline{77.65} & \underline{89.53} & \underline{81.05}
      & \underline{91.65} \\
      & PXtal-UOT$_{\mathrm{GKL}}$
      & \textbf{70.18} & \textbf{85.39} & \textbf{89.32}
      & \textbf{78.74} & \textbf{90.22} & \textbf{82.18}
      & \textbf{91.88} \\
    \midrule
    \multirow{5}{*}{Alex-MP-20}
      & PXtal-CLIP
      & 54.48 & 69.89 & 75.41 & 63.93 & 78.83 & 70.48 & 84.14 \\
      & PXtal-OT
      & 56.87 & 72.33 & 77.53 & 66.25 & 80.62 & 72.41 & 85.71 \\
      & PXtal-FGOT
      & 56.39 & 72.03 & 77.13 & 65.82 & 80.07 & 71.98 & 85.61 \\
      & PXtal-UFGOT
      & \underline{58.51} & \underline{74.44} & \underline{79.65}
      & \underline{68.11} & \textbf{83.99} & \underline{75.37}
      & \textbf{88.90} \\
      & PXtal-UOT$_{\mathrm{GKL}}$
      & \textbf{59.24} & \textbf{74.75} & \textbf{79.79}
      & \textbf{68.60} & \underline{83.94} & \textbf{75.91}
      & \underline{88.85} \\
    \bottomrule
  \end{tabular*}
  \caption{Zero-shot PXRD-to-crystal retrieval on JARVIS-DFT-3D and
    Alex-MP-20. Models rank each target dataset's candidate library without
    target-set training or fine tuning. Metrics are in percent. Bold and
    underline mark the best and second-best models within each dataset.}
  \label{tab:transfer_full_inorganic}
\end{table*}

\begin{table*}[!htbp]
  \centering
  \small
  \setlength{\tabcolsep}{1mm}
  \begin{tabular*}{\textwidth}{@{\extracolsep{\fill}}ll*{7}{r}@{}}
    \toprule
    & & \multicolumn{4}{c}{\textbf{Exact retrieval}}
      & \multicolumn{1}{c}{\textbf{Crystal system}}
      & \multicolumn{2}{c}{\textbf{Space group}} \\
    \cmidrule(lr){3-6}\cmidrule(lr){7-7}\cmidrule(lr){8-9}
    \textbf{Dataset} & \textbf{Model}
      & Acc@1$\uparrow$ & Acc@3$\uparrow$ & Acc@5$\uparrow$ & MRR$\uparrow$
      & Acc@1$\uparrow$ & Acc@1$\uparrow$ & Acc@3$\uparrow$ \\
    \midrule
    \multirow{5}{*}{QMOF}
      & PXtal-CLIP
      & 11.09 & 20.20 & 25.21 & 18.53 & 53.12 & 36.33 & 60.14 \\
      & PXtal-OT
      & 16.32 & 26.93 & 32.25 & 24.56 & 51.67 & 37.87 & 62.44 \\
      & PXtal-FGOT
      & \underline{19.46} & \underline{31.93} & \underline{38.51}
      & \underline{28.83} & \underline{55.30} & \textbf{42.76}
      & \underline{66.62} \\
      & PXtal-UFGOT
      & 18.56 & 30.63 & 36.40 & 27.52 & 54.64 & 41.75 & 64.33 \\
      & PXtal-UOT$_{\mathrm{GKL}}$
      & \textbf{20.99} & \textbf{33.51} & \textbf{39.54}
      & \textbf{30.10} & \textbf{55.87} & \underline{42.54}
      & \textbf{67.45} \\
    \midrule
    \multirow{5}{*}{Perov-5}
      & PXtal-CLIP
      & 19.31 & 40.63 & 52.31 & 34.96 & 57.99 & 57.99
      & \underline{88.38} \\
      & PXtal-OT
      & 20.18 & 42.99 & 55.35 & 36.23 & 60.05 & 60.03 & 86.26 \\
      & PXtal-FGOT
      & 21.74 & 43.83 & 55.40 & 37.02 & 59.34 & 59.34
      & \textbf{88.93} \\
      & PXtal-UFGOT
      & \underline{23.46} & \underline{45.52} & \underline{57.20}
      & \underline{39.15} & \underline{62.99} & \underline{62.96}
      & 88.11 \\
      & PXtal-UOT$_{\mathrm{GKL}}$
      & \textbf{25.63} & \textbf{49.67} & \textbf{61.93}
      & \textbf{41.87} & \textbf{64.23} & \textbf{64.23}
      & 88.32 \\
    \bottomrule
  \end{tabular*}
  \caption{Zero-shot PXRD-to-crystal retrieval on QMOF and Perov-5. Models
    rank each target dataset's candidate library without target-set training
    or fine tuning. Metrics are in percent. Bold and underline mark the best
    and second-best models within each dataset.}
  \label{tab:transfer_full_frameworks}
\end{table*}

\subsection{Retrieval across PXRD Neighborhoods}
Table~\ref{tab:jarvis_neighborhood_p2c} lists the Acc@1/5/10 and MRR values
used in the quartile analysis in the main text.

\begin{table*}[!htbp]
  \centering
  \footnotesize
  \setlength{\tabcolsep}{0.8mm}
  \begin{tabular*}{\textwidth}{@{\extracolsep{\fill}}lrrrr@{\hspace{4mm}}lrrrr@{}}
    \toprule
    & \multicolumn{4}{c}{\textbf{Q1: lowest similarity}}
    & & \multicolumn{4}{c}{\textbf{Q2}} \\
    \cmidrule(lr){2-5}\cmidrule(lr){7-10}
    \textbf{Model} & Acc@1$\uparrow$ & Acc@5$\uparrow$ & Acc@10$\uparrow$
      & MRR$\uparrow$
    & \textbf{Model} & Acc@1$\uparrow$ & Acc@5$\uparrow$ & Acc@10$\uparrow$
      & MRR$\uparrow$ \\
    \midrule
    PXtal-CLIP & 70.45 & 88.29 & 92.42 & 78.34
      & PXtal-CLIP & 66.66 & 86.58 & 91.00 & 75.63 \\
    PXtal-OT & 74.13 & 91.16 & 94.58 & 81.69
      & PXtal-OT & 69.79 & 89.21 & 92.71 & 78.47 \\
    PXtal-FGOT & 74.13 & 91.50 & 94.74 & 81.83
      & PXtal-FGOT & 69.61 & 88.53 & 91.95 & 78.06 \\
    PXtal-UFGOT & \underline{77.11} & \underline{93.53}
      & \underline{96.24} & \underline{84.41}
      & PXtal-UFGOT & \underline{73.42} & \underline{91.84}
      & \underline{94.82} & \underline{81.50} \\
    PXtal-UOT$_{\mathrm{GKL}}$ & \textbf{78.00} & \textbf{94.32}
      & \textbf{96.66} & \textbf{85.19}
      & PXtal-UOT$_{\mathrm{GKL}}$ & \textbf{74.42} & \textbf{92.42}
      & \textbf{95.29} & \textbf{82.56} \\
    \midrule
    & \multicolumn{4}{c}{\textbf{Q3}}
    & & \multicolumn{4}{c}{\textbf{Q4: highest similarity}} \\
    \cmidrule(lr){2-5}\cmidrule(lr){7-10}
    \textbf{Model} & Acc@1$\uparrow$ & Acc@5$\uparrow$ & Acc@10$\uparrow$
      & MRR$\uparrow$
    & \textbf{Model} & Acc@1$\uparrow$ & Acc@5$\uparrow$ & Acc@10$\uparrow$
      & MRR$\uparrow$ \\
    \midrule
    PXtal-CLIP & 64.24 & 83.53 & 88.00 & 73.03
      & PXtal-CLIP & 45.29 & 62.95 & 66.95 & 53.33 \\
    PXtal-OT & 66.74 & 86.68 & 91.26 & 75.68
      & PXtal-OT & 49.24 & 68.50 & 72.66 & 58.07 \\
    PXtal-FGOT & 67.45 & 86.16 & 90.29 & 75.86
      & PXtal-FGOT & 50.18 & 68.53 & 73.03 & 58.65 \\
    PXtal-UFGOT & \underline{71.55} & \underline{91.05}
      & \underline{94.82} & \underline{80.24}
      & PXtal-UFGOT & \underline{53.47} & \underline{78.05}
      & \underline{83.47} & \underline{64.47} \\
    PXtal-UOT$_{\mathrm{GKL}}$ & \textbf{72.74} & \textbf{92.00}
      & \textbf{95.24} & \textbf{81.27}
      & PXtal-UOT$_{\mathrm{GKL}}$ & \textbf{55.61} & \textbf{78.53}
      & \textbf{83.92} & \textbf{65.97} \\
    \bottomrule
  \end{tabular*}
  \caption{PXRD-to-crystal retrieval on JARVIS-DFT-3D across PXRD neighborhood
    similarity quartiles. Q1 contains the neighborhoods with the lowest
    similarity and Q4 those with the highest. Metrics are reported in percent.
    Bold and underline indicate the best and second-best models within each
    quartile, respectively.}
  \label{tab:jarvis_neighborhood_p2c}
\end{table*}

\subsection{Crystal and PXRD Encoder Transfer}
\label{app:downstream_full}

The main text defines the downstream tasks and the probe and fine-tuning
protocols. Table~\ref{tab:pxrd_encoder_probe} adds MP-100 probe results for
PXRD encoders jointly pretrained with three crystal encoders.
Table~\ref{tab:matbench_main_stats} gives the Matbench metrics, their
variation across folds, and RMSE for the four regression tasks.

\begin{table*}[!htbp]
  \centering
  \small
  \setlength{\tabcolsep}{0.9mm}
  \begin{tabular}{@{}ll*{2}{c}@{\hspace{1.2mm}}*{2}{c}
    @{\hspace{1.2mm}}*{2}{c}@{\hspace{1.2mm}}*{2}{c}
    @{\hspace{1.2mm}}*{2}{c}@{}}
    \toprule
    \shortstack{\textbf{Crystal}\\\textbf{encoder}} & \textbf{Model}
      & \multicolumn{2}{c}{\textbf{Crystal system}}
      & \multicolumn{2}{c}{\textbf{Space group}}
      & \multicolumn{2}{c}{\textbf{Lattice lengths}}
      & \multicolumn{2}{c}{\textbf{Lattice angles}}
      & \multicolumn{2}{c}{\textbf{Unit-cell volume}} \\
    \cmidrule(lr){3-4}\cmidrule(lr){5-6}\cmidrule(lr){7-8}
    \cmidrule(lr){9-10}\cmidrule(lr){11-12}
      & & Acc. (\%)$\uparrow$ & F1 (\%)$\uparrow$
      & Acc. (\%)$\uparrow$ & F1 (\%)$\uparrow$
      & MAE (\AA)$\downarrow$ & $R^2\uparrow$
      & MAE ($^\circ$)$\downarrow$ & $R^2\uparrow$
      & MAE (\AA$^3$)$\downarrow$ & $R^2\uparrow$ \\
    \midrule
    \multirow{4}{*}{CSPNet}
      & PXtal-CLIP
      & 76.37 & 75.35 & 64.71 & 46.28
      & 0.80 & 0.76 & 6.49 & 0.49 & 103.50 & 0.77 \\
      & PXtal-OT
      & 79.14 & 77.98 & 67.13 & \underline{47.82}
      & \underline{0.74} & \underline{0.78}
      & \underline{6.20} & 0.50
      & 97.57 & \underline{0.78} \\
      & PXtal-UOT$_{\mathrm{CE}}$
      & \underline{80.22} & \underline{79.16} & \underline{67.42} & 47.47
      & \underline{0.74} & \textbf{0.79}
      & 6.24 & \underline{0.51}
      & \underline{95.92} & \textbf{0.80} \\
      & PXtal-UOT$_{\mathrm{GKL}}$
      & \textbf{81.33} & \textbf{80.39} & \textbf{69.50} & \textbf{51.40}
      & \textbf{0.72} & \textbf{0.79}
      & \textbf{5.95} & \textbf{0.53}
      & \textbf{93.37} & \textbf{0.80} \\
    \midrule
    \multirow{4}{*}{Matformer}
      & PXtal-CLIP
      & \underline{76.95} & \underline{75.31} & 63.34 & 43.81
      & \underline{0.80} & \underline{0.75}
      & \underline{6.52} & \underline{0.48}
      & 104.08 & \underline{0.76} \\
      & PXtal-OT
      & \textbf{77.69} & \textbf{76.56} & \underline{63.92} & \underline{44.13}
      & \underline{0.80} & \textbf{0.76}
      & 6.60 & 0.46
      & \underline{103.57} & \textbf{0.77} \\
      & PXtal-UOT$_{\mathrm{CE}}$
      & 75.88 & 74.64 & 62.41 & 42.71
      & \underline{0.80} & \underline{0.75}
      & 6.53 & 0.47
      & 103.95 & \textbf{0.77} \\
      & PXtal-UOT$_{\mathrm{GKL}}$
      & 76.67 & 75.21 & \textbf{64.11} & \textbf{45.15}
      & \textbf{0.78} & \textbf{0.76}
      & \textbf{6.45} & \textbf{0.49}
      & \textbf{102.34} & \textbf{0.77} \\
    \midrule
    \multirow{4}{*}{GemNet-T}
      & PXtal-CLIP
      & \underline{74.08} & 73.04 & \underline{59.65} & 41.02
      & \underline{0.97} & \underline{0.68}
      & \underline{7.23} & \underline{0.42}
      & 121.66 & \underline{0.70} \\
      & PXtal-OT
      & 73.19 & 71.96 & 58.89 & 39.60
      & 0.98 & 0.67
      & 7.29 & 0.41
      & 123.72 & 0.68 \\
      & PXtal-UOT$_{\mathrm{CE}}$
      & \underline{74.08} & \underline{73.15} & 59.57 & \underline{41.16}
      & \underline{0.97} & \underline{0.68}
      & 7.31 & \underline{0.42}
      & \textbf{120.85} & \textbf{0.71} \\
      & PXtal-UOT$_{\mathrm{GKL}}$
      & \textbf{74.79} & \textbf{73.86} & \textbf{60.52} & \textbf{42.38}
      & \textbf{0.96} & \textbf{0.69}
      & \textbf{7.15} & \textbf{0.43}
      & \underline{121.33} & \underline{0.70} \\
    \bottomrule
  \end{tabular}
  \caption{MP-100 PXRD encoder probe results. Row groups identify the crystal
    encoder used during joint pretraining. Columns report symmetry
    classification and unit-cell regression. Bold and underline mark the best
    and second-best models within each crystal-encoder group; ranks use
    unrounded values.}
  \label{tab:pxrd_encoder_probe}
\end{table*}

\begin{table*}[!htbp]
  \centering
  \small
  \setlength{\tabcolsep}{0.9mm}
  \begin{tabular*}{\textwidth}{@{\extracolsep{\fill}}l*{6}{c}@{}}
    \toprule
    \textbf{Model}
      & \shortstack{\textbf{Band gap}\\MAE (eV)}
      & \shortstack{\textbf{Formation energy}\\MAE (eV/atom)}
      & \shortstack{\textbf{Metal}\\AUROC}
      & \shortstack{\textbf{Metal}\\F1}
      & \shortstack{\textbf{Bulk modulus}\\MAE $\log_{10}(\mathrm{GPa})$}
      & \shortstack{\textbf{Shear modulus}\\MAE $\log_{10}(\mathrm{GPa})$} \\
    \midrule
    PXtal-CLIP
      & $0.1911\pm0.0021$ & $0.0231\pm0.0005$
      & $0.9045\pm0.0009$ & $0.8924\pm0.0010$
      & $0.0545\pm0.0031$ & $0.0729\pm0.0013$ \\
    PXtal-OT
      & $0.1882\pm0.0026$ & $0.0229\pm0.0005$
      & $0.9054\pm0.0019$ & $0.8933\pm0.0019$
      & $0.0530\pm0.0028$ & $0.0720\pm0.0015$ \\
    PXtal-FGOT
      & $0.1881\pm0.0011$ & $0.0228\pm0.0005$
      & $0.9039\pm0.0038$ & $0.8914\pm0.0040$
      & $0.0532\pm0.0029$ & $0.0716\pm0.0019$ \\
    PXtal-UFGOT
      & $0.1885\pm0.0025$ & $0.0228\pm0.0004$
      & $0.9040\pm0.0018$ & $0.8917\pm0.0020$
      & $0.0539\pm0.0025$ & $0.0727\pm0.0012$ \\
    PXtal-UOT$_{\mathrm{GKL}}$
      & $0.1885\pm0.0033$ & $0.0228\pm0.0006$
      & $0.9036\pm0.0015$ & $0.8913\pm0.0019$
      & $0.0538\pm0.0028$ & $0.0715\pm0.0010$ \\
    \midrule
    From scratch
      & $0.2155\pm0.0023$ & $0.0245\pm0.0004$
      & $0.8969\pm0.0026$ & $0.8837\pm0.0030$
      & $0.0643\pm0.0025$ & $0.0829\pm0.0017$ \\
    \bottomrule
  \end{tabular*}

  \medskip
  \setlength{\tabcolsep}{2mm}
  \begin{tabular*}{\textwidth}{@{\extracolsep{\fill}}l*{4}{c}@{}}
    \toprule
    \textbf{Model}
      & \shortstack{\textbf{Band gap}\\RMSE (eV)}
      & \shortstack{\textbf{Formation energy}\\RMSE (eV/atom)}
      & \shortstack{\textbf{Bulk modulus}\\RMSE $\log_{10}(\mathrm{GPa})$}
      & \shortstack{\textbf{Shear modulus}\\RMSE $\log_{10}(\mathrm{GPa})$} \\
    \midrule
    PXtal-CLIP
      & $0.4803\pm0.0062$ & $0.0597\pm0.0044$
      & $0.1085\pm0.0077$ & $0.1144\pm0.0032$ \\
    PXtal-OT
      & $0.4771\pm0.0117$ & $0.0594\pm0.0047$
      & $0.1069\pm0.0062$ & $0.1134\pm0.0035$ \\
    PXtal-FGOT
      & $0.4740\pm0.0063$ & $0.0600\pm0.0059$
      & $0.1080\pm0.0065$ & $0.1132\pm0.0033$ \\
    PXtal-UFGOT
      & $0.4736\pm0.0109$ & $0.0594\pm0.0053$
      & $0.1086\pm0.0059$ & $0.1145\pm0.0033$ \\
    PXtal-UOT$_{\mathrm{GKL}}$
      & $0.4751\pm0.0122$ & $0.0599\pm0.0052$
      & $0.1082\pm0.0068$ & $0.1128\pm0.0025$ \\
    \midrule
    From scratch
      & $0.5255\pm0.0102$ & $0.0620\pm0.0061$
      & $0.1208\pm0.0065$ & $0.1284\pm0.0040$ \\
    \bottomrule
  \end{tabular*}
  \caption{Five-fold Matbench crystal encoder transfer. The upper block gives
    the official MAE and AUROC metrics and also reports F1 for metal
    classification. The lower block adds RMSE for the four regression tasks.
    Values are mean $\pm$ standard deviation over the official folds.}
  \label{tab:matbench_main_stats}
  \label{tab:matbench_rmse}
\end{table*}

\FloatBarrier
\section{Reproducibility Details}\label{app:repro}
\label{app:evaluation_protocol}

\begin{table*}[!htbp]
  \centering
  \small
  \setlength{\tabcolsep}{1.5mm}
  \begin{tabular*}{\textwidth}{@{\extracolsep{\fill}}lrrrrl@{}}
    \toprule
    \textbf{Dataset} & \textbf{Train} & \textbf{Validation} & \textbf{Test}
      & \textbf{Total} & \textbf{PXRD grid} \\
    \midrule
    MP-100
      & 28,952 & 6,204 & 6,203 & 41,359
      & $2001\!\to\!2048$ points, $5^\circ$--$65^\circ$ \\
    CrystaLLM-v1
      & 2,047,874 & 227,542 & 10,286 & 2,285,702
      & 4096 points, $5^\circ$--$90^\circ$ \\
    \bottomrule
  \end{tabular*}
  \caption{Paired PXRD and crystal-structure datasets. Counts use the train,
    validation, and test splits reported throughout the paper. For MP-100,
    the arrow runs from the GSAS-II source grid to the linearly resampled
    model grid.}
  \label{tab:dataset_summary}
\end{table*}

\paragraph{MP-100 construction and PXRD simulation.}
MP-100 draws from Materials Project release \texttt{v2025.06.09}
\citep{jain2013materialsproject}. Selection requires an inorganic crystal
structure with at most 100 atoms per unit cell, an experimental Inorganic
Crystal Structure Database (ICSD) reference
\citep{hellenbrandt2004inorganic}, and an energy above the convex hull below
$0.25$~eV/atom. We remove disconnected neutral molecular fragments and
species that are gases at standard conditions, including H$_2$, O$_2$, and
noble gases, then apply the training pipeline's validity checks. Random seed
42 gives the split in Table~\ref{tab:dataset_summary}. GSAS-II
\citep{toby2013} simulates each pattern with the single-wavelength protocol in
\hyperref[sec:exact_gsasii_simulation]{Appendix~\ref*{app:pxrd}}; Table~\ref{tab:dataset_summary}
summarizes the input grids and MP-100 resampling.

\paragraph{CrystaLLM-v1 construction.}
We use the supplied CrystaLLM train, validation, and test CIF files according
to fixed manifests \citep{antunes2024crystal}. Each raw CIF is simulated with
the same GSAS-II protocol before crystal-side standardization.
Table~\ref{tab:dataset_summary} gives the split and input grid. Space-group
and crystal-system metadata come from the source record when available;
otherwise, pymatgen infers them from the standardized crystal structure
\citep{ong2013python}.

\subsection{Model and Training Details}

\paragraph{PXRD encoder configuration.}
PXtal uses a dense-patch Transformer \citep{vaswani2017attention}. It divides
the normalized intensity curve into contiguous patches and uses width 512,
feed-forward width 2048, and 32 attention heads. MP-100 uses six layers and
patch size 16 on the 2048-point pattern; CrystaLLM-v1 uses 12 layers and patch
size 32 on the 4096-point pattern. Both settings produce 128 patch tokens. The
encoder uses the Gaussian error linear unit (GELU), pre-layer normalization,
absolute position embeddings, 0.1 embedding and residual dropout, 0.1
attention dropout, and no activation dropout \citep{hendrycks2016gaussian}.

\paragraph{Crystal encoder configuration.}
MP-100 pairs the PXRD encoder with three crystal encoders. CSPNet follows
DiffCSP's fully connected periodic graph network \citep{jiao2023crystal}; its
fractional-coordinate message passing uses 10 layers of width 512. Matformer
uses five lattice-aware graph
Transformer layers, four attention heads, node width 384, an 8\,\AA{} cutoff,
and at most 12 neighbors \citep{yan2022periodic}. PXtal reads the pooled
384-dimensional vector before Matformer's configured 64-dimensional property
layer. GemNet-T uses four interaction blocks, 512-dimensional atom and edge
embeddings, a 7\,\AA{} cutoff, and at most 50 neighbors. It mean-pools the
512-dimensional atom embeddings and supplies a fixed zero conditioning latent
of width 512 \citep{gasteiger2021gemnet}.
CrystaLLM-v1 pretraining uses the MatterGen implementation of GemNet-T
\citep{zeni2025mattergen}.
Each encoder output passes through a bias-free linear map to the shared
512-dimensional space and is then $\ell_2$ normalized.

\paragraph{Software environment.}
Training and evaluation used Linux nodes with NVIDIA GH200 120GB Grace
Hopper Superchips.
The software stack is Python 3.11, PyTorch 2.6.0 with CUDA 12.6
\citep{paszke2019pytorch}, PyTorch Geometric 2.6.1 \citep{fey2019fast}, and
PyTorch Lightning 2.3.3 \citep{falcon2024pytorchlightning}. We materialize Matbench v0.1 with
\texttt{matbench} v0.6 \citep{dunn2020benchmarking}.

\begin{table}[!htbp]
  \centering
  \footnotesize
  \setlength{\tabcolsep}{1.8pt}
  \renewcommand{\arraystretch}{1.02}
  \caption*{\textbf{Pretraining and optimization settings.} Batch sizes are
    per GPU; parentheses give the global size.}
  \begin{tabular}{@{}
      >{\raggedright\arraybackslash}p{0.25\columnwidth}
      >{\raggedright\arraybackslash}p{0.31\columnwidth}
      >{\raggedright\arraybackslash}p{0.33\columnwidth}@{}}
    \toprule
    \textbf{Setting} & \textbf{MP-100} & \textbf{CrystaLLM-v1} \\
    \midrule
    \multicolumn{3}{@{}l}{\textit{Shared settings}} \\
    Optimizer
      & \multicolumn{2}{l}{AdamW \citep{loshchilov2019decoupled}} \\
    $\boldsymbol{\beta}$
      & \multicolumn{2}{l}{$(0.9,0.999)$} \\
    Denom.\ constant
      & \multicolumn{2}{l}{$10^{-6}$} \\
    Schedule
      & \multicolumn{2}{l}{Cosine decay; one update per batch} \\
    Seed; selection
      & \multicolumn{2}{l}{42; best validation performance} \\
    \midrule
    \multicolumn{3}{@{}l}{\textit{Dataset-specific settings}} \\
    Epochs; GPUs
      & 300 / 2
      & 100 / 16 (DDP) \\
    Batch size
      & 64 (128)
      & 512 (8192) \\
    Learning rate
      & $5\times10^{-4}$
      & $6\times10^{-4}$ \\
    Weight decay
      & $10^{-2}$
      & $10^{-3}$ \\
    Linear warmup
      & 24 epochs
      & 12 epochs \\
    Precision
      & Full
      & bfloat16 mixed \\
    Grad.\ clip
      & None
      & Unit norm at 1 \\
    Validation
      & Every 25 epochs
      & Every 10 epochs \\
    \bottomrule
  \end{tabular}
\end{table}

\FloatBarrier
\section{Optimal Transport Algorithms and Implementation}\label{app:ot}
\begin{figure*}[!t]
  \centering
  \begin{minipage}[t]{0.94\textwidth}
    \vspace{0pt}
    \hrule
    \vspace{2pt}
    \refstepcounter{algorithm}\label{alg:uot_ufgot_fgot_solver}%
    \noindent\textbf{Algorithm \thealgorithm}\ Unified proximal solver for
      PXtal-UOT, PXtal-UFGOT, and PXtal-FGOT
    \par\vspace{2pt}
    \hrule
    \vspace{5pt}
    \small
    \begin{algorithmic}[1]
      \setlength{\itemsep}{0.5pt}
      \REQUIRE Normalized embeddings $\mathbf Z^P,\mathbf Z^C$;
        $\mathtt{method}\in\{\mathrm{UOT},\mathrm{UFGOT},\mathrm{FGOT}\}$
      \REQUIRE $\varepsilon,\alpha_0,\alpha_1,\rho,
        \tau_{\mathrm{prox}}$; iterations $R,K,M$; stabilizer $\delta$
      \STATE $N\leftarrow\operatorname{rows}(\mathbf Z^P)$,\quad
        $\mathbf a,\mathbf b\leftarrow N^{-1}\mathbf1_N$
      \STATE $\mathbf C^\Theta\leftarrow
        \mathbf1_{N\times N}-\mathbf Z^P(\mathbf Z^C)^\top$
      \STATE $\mathbf D^P\leftarrow
        \mathbf Z^P(\mathbf Z^P)^\top$,\quad
        $\mathbf D^C\leftarrow
        \mathbf Z^C(\mathbf Z^C)^\top$
      \STATE $\mathbf P^{(0)}\leftarrow\mathbf a\mathbf b^\top$
      \IF{$\mathtt{method}=\mathrm{UOT}$}
        \STATE $\rho,\tau_{\mathrm{prox}}\leftarrow0$
      \ENDIF
      \FOR{$r=0,\ldots,R-1$}
        \STATE $\mathbf C_{\mathrm{rel}}^{(r)}\leftarrow
          \mathbf C^\Theta-\rho\mathbf D^P\mathbf P^{(r)}\mathbf D^C$
        \STATE $\widehat{\mathbf C}^{(r)}\leftarrow
          \mathbf C_{\mathrm{rel}}^{(r)}
          -\tau_{\mathrm{prox}}\log[
          \operatorname{clamp}_{\delta}(\mathbf P^{(r)})]$
        \IF{$\mathtt{method}=\mathrm{FGOT}$}
          \STATE $\mathbf P^{(r+1)}\leftarrow
            \operatorname{Sinkhorn}\!\left(
            \widehat{\mathbf C}^{(r)},
            \varepsilon+\tau_{\mathrm{prox}},K\right)$
        \ELSE
          \STATE $\mathbf P^{(r+1)}\leftarrow
            \operatorname{UOTScale}\!\left(
            \widehat{\mathbf C}^{(r)},\mathbf a,\mathbf b;
            \varepsilon,\alpha_0,\alpha_1,\tau_{\mathrm{prox}},M\right)$
        \ENDIF
      \ENDFOR
      \STATE \textbf{return} $\mathbf P^{(R)}$
    \end{algorithmic}
    \vspace{2pt}
    \hrule
  \end{minipage}
\end{figure*}

\label{app:ot_losses}
\setlength{\intextsep}{8pt plus 2pt minus 2pt}

\textbf{Notation from the main text.}
We retain the notation of Sec.~\ref{sec:methods}. Superscripts $P$ and $C$ label the PXRD and
crystal modalities, $N$ is the batch size, and $d$ is the embedding width.
The two encoders have parameters $\Theta=\{\theta,\psi\}$. Unless stated
otherwise, $\mathbf a=\mathbf b=N^{-1}\mathbf1_N$ are the prescribed uniform
marginals, and $\tilde{\mathbf P}=N^{-1}\mathbf I_N$ is the diagonal target
coupling. For any coupling $\mathbf P$,
$\mathbf u=\mathbf P\mathbf1_N$ and
$\mathbf v=\mathbf P^\top\mathbf1_N$ are its row and column marginals.
The function $\mathcal H(\mathbf P)$ is the coupling entropy, and
$\mathcal U(\mathbf a,\mathbf b)$ is the set of admissible couplings with
prescribed marginals. Section~\ref{app:proofs} defines
$D_{\mathrm{GKL}}$, including its boundary values.
The vector $\mathbf1_N$ contains $N$ ones, and $\mathbf I_N$ is the
$N\times N$ identity matrix.

The entropy weight is $\varepsilon$; $\alpha_0$ and $\alpha_1$ penalize
deviations of the row and column marginals from $\mathbf a$ and $\mathbf b$;
$\rho$ weights the bilinear score; and $\tau_{\mathrm{prox}}$ weights the
proximal term. The symbol $\mathbf P^{(r)}$ denotes outer proximal iterate
$r$ for the method selected in Algorithm~\ref{alg:uot_ufgot_fgot_solver}. The method-specific names for its
final coupling are given below.

\textbf{Solver notation.}
The algorithms stack the normalized embeddings in
$\mathbf Z^P,\mathbf Z^C\in\mathbb R^{N\times d}$, whose rows are
$\mathbf z_i^P$ and $\mathbf z_i^C$.
$\mathbf1_{N\times N}$ is the all-ones matrix.
Angle brackets denote the Frobenius inner product for matrix arguments.
$R$, $K$, and $M$ are the numbers of outer proximal steps, balanced
Sinkhorn updates, and unbalanced scaling updates; $r$, $k$, and $m$ index
them.
$\boldsymbol\xi,\boldsymbol\zeta\in\mathbb R_+^N$ are the balanced Sinkhorn
scaling vectors.
$\operatorname{rows}(\cdot)$ returns the number of rows,
$\operatorname{diag}(\mathbf x)$ places $\mathbf x$ on a diagonal, and
$\oslash$ denotes elementwise division. We use
$\operatorname{LSE}_{j}(Y_{ij})=\log\sum_j\exp(Y_{ij})$, with the analogous
column-wise definition for $\operatorname{LSE}_{i}$. In the unbalanced solver,
$\mathbf A,\mathbf B\in\mathbb R^N$ are the row and column log-domain
correction vectors,
$\mathbf Y\in\mathbb R^{N\times N}$ is the current log-coupling, and
$\varepsilon_{\mathrm{eff}}=\varepsilon+\tau_{\mathrm{prox}}$ is the
effective entropy weight.
The vectors $\boldsymbol\ell_P$ and $\boldsymbol\ell_C$ hold PXRD-row and
crystal-column log-sum-exp values, and $\mathbf0$ is a zero vector of the
required length. Matrix logarithms and exponentials act entrywise; gradients
pass through every transport iteration. Before evaluating $\log\mathbf P$,
$\operatorname{clamp}_{\delta}$ floors each entry at $\delta=10^{-8}$.

\textbf{Learned cost and bilinear score.}
Algorithms~\ref{alg:uot_ufgot_fgot_solver} and~\ref{alg:uot_scaling} instantiate the main-text relation matrices
$\mathbf D^P$ and $\mathbf D^C$ as within-modality cosine Gram matrices:
\[
  \begin{aligned}
    \mathbf C^\Theta
      &=\mathbf1_{N\times N}-\mathbf Z^P(\mathbf Z^C)^\top,\\
    \mathbf D^P&=\mathbf Z^P(\mathbf Z^P)^\top,\qquad
    \mathbf D^C=\mathbf Z^C(\mathbf Z^C)^\top .
  \end{aligned}
\]
PXtal-FGOT and PXtal-UFGOT use the bilinear score and the outer-step cost
\[
  \begin{aligned}
    \mathcal B(\mathbf P;\mathbf D^P,\mathbf D^C)
      &=\langle\mathbf P,\mathbf D^P\mathbf P\mathbf D^C\rangle,\\
    \mathbf C_{\mathrm{rel}}^{(r)}
      &=\mathbf C^\Theta-\rho\mathbf D^P\mathbf P^{(r)}\mathbf D^C .
  \end{aligned}
\]
The first term is the cross-modal cosine-distance cost. The second rewards
couplings that preserve within-modality cosine similarities.

\subsection{Balanced Sinkhorn Scaling and PXtal-OT}

Algorithms~\ref{alg:sinkhorn_solver} and \ref{alg:pxtal_ot_solver} state the
balanced solver and PXtal-OT loss. With Gibbs kernel
$\mathbf K_\varepsilon=\exp(-\mathbf C/\varepsilon)$ and scaling vectors
$\boldsymbol\xi,\boldsymbol\zeta$, their updates enforce
$\mathbf P\mathbf1_N=\mathbf a$ and
$\mathbf P^\top\mathbf1_N=\mathbf b$. For $\mathbf C^\Theta$,
Corollary~\ref{cor:conditional_ce} gives the exact balanced-coupling identity
\begin{equation}
  \begin{aligned}
    \mathbf P_\Theta^{\mathrm{OT}}
      &=\operatorname{Sinkhorn}(\mathbf C^\Theta,\varepsilon,K),\\
    \mathcal L_{\mathrm{OT}}
      &=\mathrm{KL}\!\left(
        \tilde{\mathbf P}\|\mathbf P_\Theta^{\mathrm{OT}}\right)
        =\mathcal L_{\mathrm{CE}}\!\left(
        \mathbf P_\Theta^{\mathrm{OT}}\right)
      =-\frac{1}{N}\sum_{i=1}^{N}
        \log\frac{P_{\Theta,ii}^{\mathrm{OT}}}
        {\sum_{j=1}^{N}P_{\Theta,ij}^{\mathrm{OT}}}.
  \end{aligned}
  \label{eq:balanced_solver_supp}
\end{equation}
Indices $i$ and $j$ label PXRD rows and crystal columns. The equality uses
the fixed row marginal
$\sum_jP_{\Theta,ij}^{\mathrm{OT}}=N^{-1}$.

\begin{algorithm}[H]
  \caption{Balanced Sinkhorn scaling}
  \label{alg:sinkhorn_solver}
  \small
  \begin{algorithmic}[1]
    \REQUIRE Cost $\mathbf C$, entropy weight $\varepsilon$, iterations $K$
    \STATE $N\leftarrow\operatorname{rows}(\mathbf C)$,\quad
      $\mathbf a,\mathbf b\leftarrow N^{-1}\mathbf1_N$
    \STATE $\mathbf K_\varepsilon\leftarrow
      \exp(-\mathbf C/\varepsilon)$,\quad
      $\boldsymbol\xi\leftarrow\mathbf a$,\quad
      $\boldsymbol\zeta\leftarrow\mathbf b$
    \FOR{$k=1,\ldots,K$}
      \STATE $\boldsymbol\zeta\leftarrow
        \mathbf b\oslash(\mathbf K_\varepsilon^\top\boldsymbol\xi)$
      \STATE $\boldsymbol\xi\leftarrow
        \mathbf a\oslash(\mathbf K_\varepsilon\boldsymbol\zeta)$
    \ENDFOR
    \STATE \textbf{return}
      $\operatorname{diag}(\boldsymbol\xi)\mathbf K_\varepsilon
       \operatorname{diag}(\boldsymbol\zeta)$
  \end{algorithmic}
\end{algorithm}

\begin{algorithm}[H]
  \caption{PXtal-OT objective}
  \label{alg:pxtal_ot_solver}
  \small
  \begin{algorithmic}[1]
    \REQUIRE Normalized embeddings $\mathbf Z^P,\mathbf Z^C$;
      $\varepsilon,K$
    \STATE $N\leftarrow\operatorname{rows}(\mathbf Z^P)$
    \STATE $\mathbf C^\Theta\leftarrow
      \mathbf1_{N\times N}-\mathbf Z^P(\mathbf Z^C)^\top$
    \STATE $\mathbf P_\Theta^{\mathrm{OT}}\leftarrow
      \operatorname{Sinkhorn}(\mathbf C^\Theta,\varepsilon,K)$
    \STATE \textbf{return}
      $\mathrm{KL}(\tilde{\mathbf P}\|
       \mathbf P_\Theta^{\mathrm{OT}})
       =\mathcal L_{\mathrm{CE}}(\mathbf P_\Theta^{\mathrm{OT}})$
  \end{algorithmic}
\end{algorithm}

\subsection{PXtal-UOT and PXtal-UFGOT}

PXtal-UOT$_{\mathrm{GKL}}$ and PXtal-UOT$_{\mathrm{CE}}$ use the same
log-domain generalized Sinkhorn solver for entropy-regularized UOT
\citep{chizat2018scaling,xu2023regularized}
(Algorithms~\ref{alg:uot_ufgot_fgot_solver} and
\ref{alg:uot_scaling}). Both produce
$\mathbf P_\Theta^{\mathrm{UOT}}$ with relaxed row and column marginals.
PXtal-UOT$_{\mathrm{CE}}$ applies conditional CE after row normalization;
PXtal-UOT$_{\mathrm{GKL}}$ applies
$D_{\mathrm{GKL}}(\tilde{\mathbf P}\|
\mathbf P_\Theta^{\mathrm{UOT}})$ to the full measure. They use the same
coupling, but GKL also penalizes its marginal mass.
For a coupling $\mathbf P$, the row-conditional loss is
\[
  \mathcal L_{\mathrm{CE}}(\mathbf P)
  =-\frac{1}{N}\sum_{i=1}^{N}
    \log\frac{P_{ii}}{\sum_{j=1}^{N}P_{ij}}.
\]

PXtal-UFGOT adds KL-proximal outer steps and the bilinear score from the main
paper. The score compares PXRD rows $i,i'$ with crystal columns $j,j'$,
weighted by the coupling entries that join them. Since
$\mathbf D^P$ and $\mathbf D^C$ are symmetric, the gradient is
\[
  \nabla_{\mathbf P}\mathcal B(\mathbf P;\mathbf D^P,\mathbf D^C)
  =2\mathbf D^P\mathbf P\mathbf D^C .
\]
Absorbing the factor of two into $\rho$ gives the iteration-$r$ cost
\[
  \mathbf C_{\mathrm{rel}}^{(r)}
  =\mathbf C^\Theta-\rho\mathbf D^P\mathbf P^{(r)}\mathbf D^C .
\]

Algorithm~\ref{alg:uot_scaling} uses $M$ updates for each PXtal-UFGOT
proximal subproblem. Since
$-\mathcal H(\mathbf P)=\sum_{ij}P_{ij}(\log P_{ij}-1)$, expanding the
proximal term and dropping constants gives
\begin{equation}
  \begin{aligned}
  &-(\varepsilon+\tau_{\mathrm{prox}})\mathcal H(\mathbf P)
  +\left\langle
  \mathbf C_{\mathrm{rel}}^{(r)}
  -\tau_{\mathrm{prox}}\log\mathbf P^{(r)},
  \mathbf P\right\rangle\\
  &+\alpha_0D_{\mathrm{GKL}}(\mathbf P\mathbf1_N\|\mathbf a)
  +\alpha_1D_{\mathrm{GKL}}(\mathbf P^\top\mathbf1_N\|\mathbf b).
  \end{aligned}
  \label{eq:ufgot_proximal_expansion}
\end{equation}
The proximal weight therefore enters the cost correction
$-\tau_{\mathrm{prox}}\log\mathbf P^{(r)}$ and the effective entropy weight
$\varepsilon+\tau_{\mathrm{prox}}$.
Each log-domain step updates both correction vectors.
The proximal weight $\tau_{\mathrm{prox}}$ differs from the PXtal-CLIP
temperature $\tau$.
Table~\ref{tab:ot_hyperparams} gives all coefficient and iteration settings.

\begin{algorithm}[t]
  \caption{Log-domain unbalanced scaling}
  \label{alg:uot_scaling}
  \small
  \begin{algorithmic}[1]
    \REQUIRE Cost $\mathbf C$; priors $\mathbf a,\mathbf b$;
      $\varepsilon,\alpha_0,\alpha_1,\tau_{\mathrm{prox}}$; iterations $M$
    \STATE $\varepsilon_{\mathrm{eff}}\leftarrow
      \varepsilon+\tau_{\mathrm{prox}}$
    \STATE $\mathbf A^{(0)},\mathbf B^{(0)}\leftarrow\mathbf0$,\quad
      $\mathbf Y^{(0)}\leftarrow-\mathbf C/\varepsilon_{\mathrm{eff}}$
    \FOR{$m=0,\ldots,M-1$}
      \STATE $\ell^{(m)}_{P,i}\leftarrow
        \operatorname{LSE}_{j}(Y^{(m)}_{ij})$
      \STATE $\ell^{(m)}_{C,j}\leftarrow
        \operatorname{LSE}_{i}(Y^{(m)}_{ij})$
      \STATE $\mathbf A^{(m+1)}\leftarrow
        \dfrac{\alpha_0}
        {\alpha_0+\varepsilon_{\mathrm{eff}}}
        \left(\mathbf A^{(m)}
        +\log\mathbf a-\boldsymbol\ell_P^{(m)}\right)$
      \STATE $\mathbf B^{(m+1)}\leftarrow
        \dfrac{\alpha_1}
        {\alpha_1+\varepsilon_{\mathrm{eff}}}
        \left(\mathbf B^{(m)}
        +\log\mathbf b-\boldsymbol\ell_C^{(m)}\right)$
      \STATE $\mathbf Y^{(m+1)}\leftarrow
        -\mathbf C/\varepsilon_{\mathrm{eff}}
        +\mathbf A^{(m+1)}\mathbf1_N^\top
        +\mathbf1_N(\mathbf B^{(m+1)})^\top$
    \ENDFOR
    \STATE \textbf{return} $\exp(\mathbf Y^{(M)})$
  \end{algorithmic}
\end{algorithm}

\subsection{PXtal-FGOT}

PXtal-FGOT is the hard-marginal branch of
Algorithm~\ref{alg:uot_ufgot_fgot_solver}. It uses the same bilinear score
and step-$r$ cost
$\mathbf C_{\mathrm{rel}}^{(r)}$ as PXtal-UFGOT, but constrains the coupling to
$\mathcal U(\mathbf a,\mathbf b)$. Its proximal subproblem is
\begin{equation}
  \mathbf P^{(r+1)}
  \approx\underset{\mathbf P\in\mathcal U(\mathbf a,\mathbf b)}
  {\operatorname{arg\,min}}\quad
  \langle\mathbf C_{\mathrm{rel}}^{(r)},\mathbf P\rangle
  -\varepsilon\mathcal H(\mathbf P)
  +\tau_{\mathrm{prox}}
    D_{\mathrm{GKL}}(\mathbf P\|\mathbf P^{(r)}).
  \label{eq:fgot_proximal_subproblem}
\end{equation}
The approximation marks a finite $K$-step Sinkhorn solve in each of the $R$
outer steps. The solver uses the corrected cost
$\mathbf C_{\mathrm{rel}}^{(r)}
-\tau_{\mathrm{prox}}\log[
\operatorname{clamp}_{\delta}(\mathbf P^{(r)})]$ and entropy weight
$\varepsilon+\tau_{\mathrm{prox}}$. PXtal-UFGOT instead calls
$\operatorname{UOTScale}$ with two marginal GKL penalties. The two variants
share the outer update and differ only in the inner marginal constraints.

Both fused variants use the coefficient $\rho$ in
Table~\ref{tab:ot_hyperparams}. In the FGOT branch,
$\rho=\tau_{\mathrm{prox}}=0$ and one outer step recover PXtal-OT.

Algorithm~\ref{alg:uot_ufgot_fgot_solver} uses $\mathbf P^{(r)}$ as a method-local iterate. With the
iteration counts in Table~\ref{tab:ot_hyperparams}, its outputs are
{\small
\[
  \mathbf P_\Theta^{\mathrm{UOT}}=\mathbf P^{(1)},\quad
  \mathbf P_\Theta^{\mathrm{UFGOT}}=\mathbf P^{(R)},\quad
  \mathbf P_\Theta^{\mathrm{FGOT}}=\mathbf P^{(R)}.
\]
}
PXtal-UOT$_{\mathrm{CE}}$ and PXtal-UOT$_{\mathrm{GKL}}$ share
$\mathbf P_\Theta^{\mathrm{UOT}}$ and apply
$\mathcal L_{\mathrm{CE}}(\mathbf P_\Theta^{\mathrm{UOT}})$ and
$D_{\mathrm{GKL}}(\tilde{\mathbf P}\|
\mathbf P_\Theta^{\mathrm{UOT}})$, respectively.
PXtal-UFGOT applies
$D_{\mathrm{GKL}}(\tilde{\mathbf P}\|
\mathbf P_\Theta^{\mathrm{UFGOT}})$.
PXtal-FGOT applies
$\mathrm{KL}(\tilde{\mathbf P}\|
\mathbf P_\Theta^{\mathrm{FGOT}})
=\mathcal L_{\mathrm{CE}}(\mathbf P_\Theta^{\mathrm{FGOT}})$, as for
PXtal-OT.

\begin{center}
  \centering
  \scriptsize
  \setlength{\tabcolsep}{1.3pt}
  \renewcommand{\arraystretch}{1.02}
  \begin{tabular}{@{}lccccc@{}}
    \toprule
    \textbf{Model} & $\boldsymbol{\varepsilon}$ & $\boldsymbol{\rho}$
      & $\boldsymbol{\alpha_0=\alpha_1}$
      & $\boldsymbol{\tau_{\mathrm{prox}}}$
      & \textbf{O/I} \\
    \midrule
    \multicolumn{6}{@{}l}{\textbf{MP-100}} \\
    OT & 0.07 & -- & -- & -- & --/3 \\
    UOT$_{\mathrm{CE}}$ & 0.07 & 0 & 0.01 & 0 & 1/3 \\
    UOT$_{\mathrm{GKL}}$ & 0.07 & 0 & 0.01 & 0 & 1/3 \\
    \midrule
    \multicolumn{6}{@{}l}{\textbf{CrystaLLM-v1}} \\
    OT & 0.07 & -- & -- & -- & --/3 \\
    UOT$_{\mathrm{GKL}}$ & 0.07 & 0 & 0.01 & 0 & 1/3 \\
    FGOT & 0.07 & 0.01 & -- & 0.01 & 3/3 \\
    UFGOT & 0.07 & 0.01 & 0.01 & 0.01 & 3/3 \\
    \bottomrule
  \end{tabular}
  \captionof{table}{PXtal solver settings. O/I gives outer/inner iterations.}
  \label{tab:ot_hyperparams}
\end{center}

\end{document}